%% file: neurips_2026.tex
\documentclass{article}

\usepackage[main,final]{neurips_2026}
\usepackage{amsmath}
\usepackage{amssymb}
\usepackage{mathtools}
\usepackage{amsthm}

\usepackage{enumitem}

\DeclareMathOperator{\diag}{diag}	
\DeclareMathOperator{\rk}{rank}	
\DeclareMathOperator{\rank}{rank}	
\DeclareMathOperator{\tr}{tr}	    
\DeclareMathOperator{\vecc}{vec}	    
\DeclareMathOperator{\mat}{mat}	    

\newcommand{\RR}{\mathbb R}
\newcommand\mb{\mathbf}

\DeclarePairedDelimiter{\norm}{\lVert}{\rVert}
\DeclareMathOperator{\vect}{vec}
\newcommand{\rsearch}{r_{\text{search}}}

\newcommand{\R}{\mathbb{R}}
\newcommand{\E}{\mathbb{E}}
\newcommand{\Prob}{\mathbb{P}}

\newcommand{\fro}[1]{\|#1\|_{\mathrm F}}
\newcommand{\op}[1]{\|#1\|_{\mathrm{op}}}

\newcommand{\maxn}[1]{\|#1\|_{\infty}}
\newcommand{\four}[1]{\|#1\|_{4}}

\newcommand{\eps}{\varepsilon}
\newcommand{\Om}{\Omega}
\newcommand{\Omc}{\Omega^{\mathsf c}}
\newcommand{\Mstar}{M^{\star}}
\newcommand{\Mdag}{\widetilde M}
\newcommand{\Ber}{\mathrm{Ber}}

\newcommand{\cE}{\mathcal E}

\newcommand{\Xdag}{\widetilde X}
\newcommand{\A}{\mathcal A}
\newcommand{\Sset}{\mathcal S}
\newcommand{\Mset}{\mathcal M}
\newcommand{\cH}{\mathcal H}

\newcommand{\cU}{\mathcal U}
\newcommand{\iid}{\mathrm{id}}

\theoremstyle{plain}
\newtheorem{theorem}{Theorem}[section]
\newtheorem{proposition}[theorem]{Proposition}
\newtheorem{lemma}[theorem]{Lemma}
\newtheorem{corollary}[theorem]{Corollary}
\theoremstyle{definition}
\newtheorem{definition}[theorem]{Definition}

\theoremstyle{remark}
\newtheorem{remark}[theorem]{Remark}

\usepackage[utf8]{inputenc} 
\usepackage[T1]{fontenc}    
\usepackage{hyperref}       
\hypersetup{hidelinks}
\usepackage{url}            
\usepackage{booktabs}       
\usepackage{amsfonts}       
\usepackage{nicefrac}       
\usepackage{microtype}      
\usepackage{xcolor}         
\usepackage{graphicx}
\usepackage{subfigure}

\usepackage{algorithm}
\usepackage{algorithmicx}
\usepackage{algpseudocode}

\title{Improving the Optimization Landscape of Matrix Completion with $\epsilon$-close Surrogates}

\author{%
  Ziye Ma \\
  Department of Computer Science\\
  City University of Hong Kong\\
  \texttt{ziyema@cityu.edu.hk} \\
}

\begin{document}

\maketitle

\begin{abstract}
Non-convex optimization is central to machine learning and AI, yet unlike its convex counterpart, it still lacks a general and principled framework. In this work, we try to highlight the importance of combining noisy perturbations with over-parametrization mathematically under the classic setting of matrix completion (MC). Matrix completion is an important non-convex recovery problem in which only a subset of a matrix’s entries is observed, and the remaining entries must be recovered using low-rank constraints. This problem is notoriously difficult because observations are limited. However, we show that by injecting $\epsilon$-level perturbations into the nullspace of the observation mask, we can construct noisy matrix sensing surrogates whose solutions remain $\mathcal{O}(\epsilon)$-close to the ground truth with valid restricted isometry property (RIP) constants. The existence of RIP further enables powerful tensor frameworks to be applied with theoretical guarantees. This approach reflects a broader empirical trend in which stochasticity or noise improves the curvature of the optimization landscape, allowing over-parameterized (a.k.a large) models to better separate signal from noise. Assuming each matrix entry is observed independently, we establish quantitative recovery guarantees without explicit requirements on sampling rate or incoherence, and discuss how such assumptions can further strengthen our results.
\end{abstract}

\section{Introduction}
Unlike convex problems, non-convex optimization landscapes are characterized by multiple local minima, saddle points, and regions of flat curvature, complicating the search for global optima \citep{jain2017nonconvex}. Consequently, a more in-depth and comprehensive study of such landscapes is essential. One popular way of doing so is to focus on mathematically structured non-convex problems like low-rank matrix recovery. Such problems have garnered increased attention due to their potential to provide deep insights. Low-rank matrix tasks, including matrix completion (MC) and matrix sensing (MS), are important problems since their applications span from collaborative filtering in recommendation systems \citep{koren2009matrix}, motion detection \citep{fattahi2020exact}, power system state estimation \citep{zhang2017conic, jin2019towards}, to image recovery \citep{gu2014weighted} and biomedical imaging \citep{lustig2008compressed}. More significantly, as these frameworks can encapsulate any polynomial optimization problem \citep{molybog2020conic} and are equivalent to training two-layer quadratic neural networks \citep{li2018algorithmic}, their theoretical impact extends well beyond their direct applications in the broader machine learning community.

Matrix sensing generally involves recovering a matrix from a set of linear measurements, formulated as:
\begin{equation}\label{eq:main_ms}
	 \min_{ X\in\mathbb{R}^{n\times \rsearch}}~ \frac{1}{2} \| \mathcal{A}( XX^T) - b \|_2^2 \coloneqq f(XX^T) \quad \text{(MS)}
\end{equation}
For the sake of convenience, we also denote $h(X) \coloneqq f(XX^\top)$. Here, $\mathcal{A}$ acts on the low-rank matrix $XX^T$ (rank bounded by $\rsearch$ by construct) and compares it to a vector of observations $b = \mathcal{A}(M^*)$, with $M^*$ being the rank-$r$ ground truth matrix of interest. $\mathcal{A}(\cdot): \RR^{n \times n} \mapsto \RR^m$ is a linear function defined as $\mathcal{A}(M)= [ \langle A_1, M \rangle, \dots, \langle A_m, M \rangle]^T$ where $\langle A_i, M \rangle \coloneqq \tr(A_i^\top M)$ and the sensing matrices $\{A_i\}_{i=1}^m$ are given. For simplicity, we assume $\rsearch = r$. The matrix completion challenge, a special case of matrix sensing, is given by:
\begin{equation}\label{eq:main_mc}
    \min_{ X\in\mathbb{R}^{n\times r}}~ \frac{1}{2} \| \mathcal{A}_\Omega (XX^T  - M^*) \|^2_2 \quad \text{(MC)}
\end{equation}
where $\Omega \subseteq [n] \times [n]$ represents the observed entries of an $n \times n$ matrix. We use the notation $N_\Omega$ to denote the matrix where $(N_\Omega)_{i,j} = N_{i,j} \cdot \mathbf{1}_{(i,j) \in \Omega}$ for any arbitrary $N \in \RR^{n \times n}$. $\mathcal{A}_\Omega( \cdot)$ is used to specifically denote the sensing operator of the matrix completion problem, where $\mathcal{A}_\Omega(M) = \vecc(M_\Omega)$. Matrix completion is distinguished by its reliance on the sample rate and the matrix's incoherence parameters. These parameters dictate the spread of matrix information across its entries and singular vectors \citep{candes2010power}. This dependency complicates matrix completion compared to matrix sensing, where challenges are often more tractable due to properties like Restricted Isometry Property (RIP) (see definition \ref{def:rip}). Below, we introduce RSS and RSC (equivalent to RIP) instead for more notational flexibility:
\begin{definition}[Restricted Strong Smoothness (RSS) and Restricted Strong Convexity (RSC)] \label{def:rsc_rss}
	Function $f$ defined in \eqref{eq:main_ms} satisfies the $(L_s,r)$-RSS property and the $(\alpha_s,r)$-RSC property if
	\begin{align*}
		&f(M) - f(N) \leq \langle M - N, \nabla f(N) \rangle + \frac{L_s}{2} \|M-N\|^2_F \\
		&f(M) - f(N) \geq \langle M - N, \nabla f(N) \rangle + \frac{\alpha_s}{2} \|M-N\|^2_F
	\end{align*}
are satisfied, respectively for all $M,N \in \mathbb{R}^{n\times n}$ with $\rk(M), \rk(N) \leq r$. Note that RSS and RSC provide a more expressible way to represent the RIP property \cite{recht2010guaranteed}, with $\delta_{r} = (L_s-\alpha_s)/(L_s+\alpha_s)$ after rescaling the sensing operator. For this quadratic loss, $f(M)-f(N)-\langle M-N,\nabla f(N)\rangle=\frac12\|\mathcal A(M-N)\|_2^2$. Thus the two conditions bound the measurement energy of admissible matrix differences from above and below.
\end{definition}

Matrix completion problems, on the other hand, all have missing entries, which means that they could not have a valid RSC constant since the null-space of $\mathcal{A}_\Omega(\cdot)$ is always non-trivial. This is why incoherence condition was proposed to use as a metric to guarantee recovery of completion problems: 
\begin{definition}[$\mu_0$-incoherence] \citep{ge2017no} \label{def:incoherence}
    Given a rank-$r$ matrix $M \in \RR^{n_1 \times n_2}$, we say it is $\mu_0$-incoherent if its truncated SVD decomposition $U\Sigma V^\top$ satisfies
    \[
    	\|e_i^\top U\|_2 \leq \sqrt{\mu_0 r/n_1}, \ \ \|e_j^\top V\|_2 \leq \sqrt{\mu_0 r/n_2}
    \]
    $\forall i, j \in [n_1], [n_2]$, where $e_i$ is the $i$-th standard basis of $\RR^{n_1}$ and $e_j$ is the $j$-th standard basis of $\RR^{n_2}$.
\end{definition}

Since $\mu_0$ is hard to gauge prior to solving this problem, applying matrix completion guarantees can be more challenging than matrix sensing problems with valid RIP (therefore RSC and RSS) constants. 


\subsection{Related Works}

\begin{table*}
\centering
\renewcommand*{\arraystretch}{1.5}
\begin{tabular}{ll}
\toprule
\multicolumn{1}{l}{Approach} & \multicolumn{1}{l}{Conditions for recovery w.h.p.} \\\midrule
Semi-definite Relaxation \citep{ding2020leave} & $p \geq \mathcal{O}\left(\mu_0 r \log(\mu_0 r) n \log(n) / n^2 \right)$ \\
BM Factorization \citep{zilber2022gnmr} & $p \geq \mathcal{O}\left(\mu_0 r \max\{\log n,\, \mu_0 r \kappa^2\}/n \right)$ \\
Nuclear Norm Lasso \citep{chen2020noisy} & $p \geq \mathcal{O}\left(\mu^2_0 \log(n)^3/n \right)$ \\
Spectral Methods \citep{chen2021spectral} & $p \geq \mathcal{O}\left(\kappa^2 \mu_0 r \log(n)/n \right)$ \\
Ours & Localized identifiability and shell regularity\\
\bottomrule
\end{tabular}
\caption{Brief summary of state-of-the-art recovery guarantees of different approaches in solving \eqref{eq:main_mc}. This assumes each entry of $M^*$ is sampled uniformly with probability $p$. $\mu_0$ and $\kappa$ denotes incoherence and condition number for $M^*$ respectively.}
\label{table:results_comparison}
\vspace{-1em}
\end{table*}

\subsubsection*{Recovery Guarantees}

The foundational work by \cite{candes2009exact} established that exact matrix recovery is possible from few entries, requiring a sample size of $\mu_0 n^{1.2}r\log(n)$ for $n \times n$ matrices of rank $r$ with incoherence parameter $\mu_0$. Enhancements in recovery guarantees and computational efficiencies followed, including spectral-gradient descent algorithms by \cite{keshavan2010matrix} and deeper insights into incoherence by \cite{candes2010power}. Studies by \cite{recht2011simpler} and \cite{gross2011recovering} expanded on these by demonstrating successful recovery without uniform random sampling, while \cite{ding2020leave} refined sampling orders further to $\mu_0 r \log(\mu_0 r) n \log(n)$. Another notable approach for the classic matrix completion problem \eqref{eq:main_mc} is the non-convex Burer-Monteiro (BM) factorized formulation \citep{ge2016matrix, ge2017no, du2017gradient}, generally requiring $p \ge \text{poly}(\kappa, r, \mu_0, \log n)/n$, with recent works like \cite{zilber2022gnmr} offering progressively sharper bounds under this framework. Other formulations include the convex matrix lasso that uses nuclear norm penalization \citep{chen2020noisy} and elementary spectral methods \citep{chen2021spectral}. If we assume that each entry of the desired ground truth matrix $M^*$ is sampled independently with probability $p$, we could then summarize the different requirements for successful recovery with high probability for these techniques in Table~\ref{table:results_comparison}. Note that we only included state-of-the-art recovery guarantees. From a high-level, existing requirements of $p$ are only described up to an unspecified constant, so the true $p$ needed for successful recovery is hard to gauge. Our bounds instead express recovery through localized identifiability and shell regularity, complementing these sampling-based guarantees. 


Additionally, we introduce some other popular techniques for matrix completion, including greedy algorithms \citep{lee2010admira,wang2014rank}, alternating minimization \citep{haldar2009rank,tanner2016low,wen2012solving}, iteratively reweighted least squares(IRLS) methods \citep{mohan2012iterative,radhakrishnan2025linear}, iterative thresholding \cite{klopp2015matrix} and Riemannian optimization \citep{mishra2014fixed,dai2010set}—reviewed comprehensively in \cite{nguyen2019low}. Another interesting line of work converts inductive matrix completion into regular MC \citep{ghassemi2018global,zilber2022inductive}, but they have extra information matrices $A,B$ not considered in \eqref{eq:main_mc}, so it's hard to compare this work to the classic setting.

\subsubsection*{Why Over-parametrization is Important}

Recent studies have highlighted over-parametrization as a crucial strategy in matrix sensing when RIP constants $\delta_{2r}$ are suboptimal (i.e., $\delta_{2r} \geq 1/2$). Research by \cite{zhang2021sharp,zhang2022improved} examined cases where the search rank $\rsearch$ exceeds the true rank $r$, thus increasing the problem's parametrization. \cite{zhang2022improved} demonstrated that for $r_{\text{search}} > r[(1+\delta_n)/(1-\delta_n)-1]^2/4$ and $r \leq \rsearch < n$, each solution $\hat X$ satisfies $\hat X \hat X^\top = M^*$. Similarly, \cite{ma2022global} established analogous results under RIP-type conditions for the $\ell_1$ loss.  In addition to the classic over-parametrization of increasing $\rsearch$, convex relaxation \citep{recht2010guaranteed} is certainly one popular approach by increasing parameters from $\mathcal{O}(nr)$ to $\mathcal{O}(n^2)$. \cite{yalcin2022semidefinite} showed that the RIP threshold for exact recovery using SDP can approach 1 when $M^*$ has a high true rank, thus underscoring the efficacy of over-parametrization. Nevertheless, the practical applicability of these conditions is limited, leading \cite{ma2023over} to explore an alternative approach to over-parametrization by lifting the search space of \eqref{eq:main_ms} into general tensor space and bank on important concepts from Sums-of-Squares optimization \citep{parrilo2003semidefinite,lasserre2001global} to convert spurious local minimizers of \eqref{eq:main_ms} into strict saddle points in the lifted space so that they could be escaped by modern optimizers. Despite its utility in resolving spurious solutions, this tensor approach's applicability to matrix completion remains constrained by the need for a valid RIP constant.

\subsection{Notation}

Scalar values such as $\sigma_i(M)$ and $\lambda_i(M)$ represent the $i$-th largest singular value and eigenvalue of matrix $M$, respectively. The inner product between two matrices $\langle A, B \rangle$ is defined as $\tr(A^\top B)$. The Euclidean norm of a vector $v$ is denoted as $\norm{v}$, while $\norm{M}_F$ and $\norm{M}_2$ are used for the Frobenius and induced $l_2$ norms of a matrix $M$; $\maxn{M}:=\max_{i,j}|M_{ij}|$ denotes its entrywise maximum norm. Vectorization $\vecc(M)$ stacks the columns of $M$ into a vector, where $\mat(v)$ is its reverse operation. The set of integers from $1$ to $n$ is expressed as $[n]$. Moreover, $\circ l$ indicates a repeated Cartesian product for $l$ times, $\oslash$ refers to the Kronecker product, and $\otimes$ signifies the tensor outer product. If these notations come with subscripts, they denote the dimension along which the operation is performed. Finally, if $S \subseteq [n] \times [m]$ represents a subset of indices of a $n \times m$ matrix, then $N_S$ refers to the masked matrix with entries $N_{ij}$ on $S$ and zeros elsewhere, and $S^{\text{c}}$ denotes its complement.

Throughout this paper, $M\succeq0$ denotes a symmetric positive semidefinite matrix. The factorization $M=XX^\top$, with $X\in\RR^{n\times r}$, restricts the candidate matrix to this cone with rank at most $r$. Orthogonal transformations $XQ$ represent the same matrix, so we measure recovery error between $XX^\top$ and $M^*$. For matrix completion, $\A_\Omega(M)=\vecc(M_\Omega)$ retains all $n^2$ coordinates, placing zeros outside $\Omega$, so $\|\A_\Omega(M)\|_2=\|M_\Omega\|_F$. The complementary masks satisfy $M=M_\Omega+M_{\Omega^c}$ and $\|M\|_F^2=\|M_\Omega\|_F^2+\|M_{\Omega^c}\|_F^2$. Thus the original and perturbed operators use observation vectors of the same dimension. We use $\widetilde M=\widetilde X\widetilde X^\top$ for a global minimizer of the surrogate, distinguishing it from the ground truth $M^*$.

For convenience, $x^{\otimes l}$ denotes the repeated tensor outer product of a vector $x$, with entries $[x^{\otimes l}]_{i_1\dots i_l}=x_{i_1}\cdots x_{i_l}$. For matrices, the Kronecker product arranges its output as a block matrix, while the tensor outer product retains the separate modes. In particular, for nonzero $X$, $\vecc(X)^{\otimes l}\in\RR^{nr\circ l}$ is a symmetric rank-one tensor even when $X$ has $r$ linearly independent columns. Hence rank-one in the lifted formulation refers to tensor rank, while the recovered matrix $XX^\top$ may have rank $r$. Tensor contractions sum over the modes indicated by their subscripts; Appendix~\ref{app:noisy_lifted} provides their coordinate definitions.

\section{The Perturbed Matrix Completion Formulation}

As explained above, most literature regarding the recovery guarantees of matrix sensing problems require some valid RIP (RSC and RSS) constant. A positive RSC constant on the whole matrix space implies a trivial nullspace, which is impossible for incomplete entry observations. To demonstrate why this is, let's consider a $2 \times 2$ matrix recovery problem, and say we observed three entries of some $M \in \RR^{2 \times 2}$ except for the lower-right entry. This will correspond to the case where
\begin{equation}\label{eq:perturb_example}
		\mathcal{A}_\Omega(M) =\vect (\begin{bmatrix}
			1 & 1 \\ 1 & 0
		\end{bmatrix} \odot M )  = \begin{bmatrix}
			1 & 0 & 0 &0 \\
			0 & 1 & 0 &0 \\
			0 & 0 & 1 &0 \\
			0 & 0 & 0 &0 
		\end{bmatrix} \vecc(M) \coloneqq T_{\Omega} \vecc(M) \in \mathbb{R}^4
\end{equation}
For this $\mathcal{A}_\Omega$ to exhibit any RIP constant $\delta < 1$, it is required that $\|\mathcal{A}_\Omega(M) \|_2^2 \geq (1-\delta) \|M\|^2_F$, meaning that the $T_{\Omega}$ matrix above cannot output $0$ unless $M$ is a zero matrix. Nevertheless, for this specific example, even if we observed three out of four entries, we can simply set $\vect(M) = [0, 0, 0, 1]^\top$ to make $\mathcal{A}_\Omega(M) = 0$, violating the RIP condition. This is a simple example showing us that RIP condition will not hold for matrix completion problems unless all entries are observed, regardless of its size. Therefore, it begs the question of \emph{whether we could use the better studied, more powerful over-parametrized MS framework to offer guarantees for MC problems?} In this work, we give a positive answer to this question via the following three steps:
\begin{enumerate}[nosep,leftmargin=*]
    \item \textbf{Surrogate Construction:} We construct a surrogate problem ($\epsilon$-MC) to solve by slightly changing $\mathcal{A}_\Omega$. $\epsilon$-MC problem is a noisy MS problem that has more favorable theoretical properties.
    \item \textbf{Surrogate Accuracy:} We prove that the surrogate's global solution approximates the true matrix $M^*$ with high probability under mild conditions.
    \item \textbf{Tensor Framework Adaptation:} To address high RIP constants in the perturbed MS setting, we extend and adapt the lifted tensor framework from \cite{ma2024algorithmic}, leveraging over-parametrization to achieve recovery.
\end{enumerate}

Despite the lack of RIP constant, a surprisingly simple solution exists. The primary issue is the zero entries in the diagonal of $T_{\Omega}$, contributing to a non-trivial nullspace. By perturbing these zero entries slightly with a small number $\epsilon \in (0,1]$, we can eliminate the nullspace. Revisiting \eqref{eq:perturb_example}, consider a perturbed sensing operator $\mathcal{A}_{\Omega,\epsilon} = \diag([1,1,1,\epsilon]) \vect(M)\coloneqq T_{\Omega,\epsilon} \vect(M)$, in which $T_{\Omega,\epsilon}$ has a trivial nullspace as promised. However, different operators lead to different observations. For example, when considering the case of \eqref{eq:perturb_example}, using $\mathcal{A}_{\Omega}$ and $\mathcal{A}_{\Omega,\epsilon}$ on the same matrix results in:
\begin{equation*}
		\mathcal{A}_{\Omega} (M) = \begin{bmatrix}
			M_{1,1}  M_{1,2} M_{2,1} 0
		\end{bmatrix}^\top \longrightarrow \mathcal{A}_{\Omega,\epsilon} (M) = \begin{bmatrix}
			M_{1,1} M_{1,2} M_{2,1} \epsilon M_{2,2}
		\end{bmatrix}^\top
	\end{equation*}
Therefore, the original observation $b$ can be considered a noisy observation under the context of $\mathcal{A}_{\Omega,\epsilon}$, with
	\begin{equation}\label{eq:w_eps_def}
	\begin{aligned}
		b = \mathcal{A}_{\Omega,\epsilon}(M^*) + w_\epsilon, \ \ w_\epsilon &= \begin{bmatrix}
			0 & 0 & 0 & -\epsilon M^*_{2,2}
		\end{bmatrix}^\top
	\end{aligned}
	\end{equation}
	where $w_\epsilon \in \RR^{n^2}$ can be considered a noise term. With this idea in place, we formally introduce our perturbed MC problem to solve:
	\begin{equation}\label{eq:eps_mc}
	\min_{ X\in\mathbb{R}^{n\times \rsearch}}~  \frac{1}{2}\| \mathcal{A}_{\Omega,\epsilon}( XX^T) - b \|_2^2 = \frac{1}{2} \fro{ \mathcal{A}_{\Om}(XX^T-\Mstar)}^2+\frac{\eps^2}{2}\fro{ \mathcal{A}_{\Omc}(XX^T)}^2 \quad \text{($\epsilon$-MC)}
\end{equation}
The optimizer receives only $\Omega$ and the observed entries $b=\A_\Omega(\Mstar)$, with zeros at missing coordinates. The complement term in \eqref{eq:eps_mc} regularizes the candidate matrix and uses no missing ground-truth values. The identity $w_\epsilon=-\epsilon\A_{\Omc}(\Mstar)$ is used only in analysis: its unknown entries cancel those of $\A_{\Omega,\epsilon}(\Mstar)$, leaving the same observed data $b$.

Equivalently, \eqref{eq:eps_mc} is a weighted low-rank approximation~\citep{RSW16,LLR16} with a deterministic weight floor on missing entries. Our contribution is the perturbation and its surrogate-accuracy certificate for symmetric PSD factors and a two-level mask; extending it to general weights and asymmetric factors is left for future work.

We denote the $\epsilon$-MC objective as $f_{w_\epsilon}(XX^\top)$, and $\Omc$ denotes the complement set to $\Omega$. Essentially, we're transforming a noiseless matrix completion problem into a noisy matrix sensing problem with operator $\mathcal{A}_{\Omega,\epsilon}$ and deterministic noise $w_\epsilon$. This approach, notably, ensures the attainment of valid RSS/RSC parameters, equivalent to RIP constants. 
\begin{lemma}[RIP/RSC of $\epsilon$-MC]\label{lem:epsilonMC-rip-general}
For $\eps\in(0,1]$, the $\eps$-MC objective satisfies $(1,n)$-RSS and $(\eps^2,n)$-RSC:
\[
\eps^2\fro{M}^2\le\|\A_{\Om,\eps}(M)\|_2^2\le\fro{M}^2
\quad\text{for all }M\in\RR^{n\times n}.
\]
If entries are observed independently with probability $p$, then for every fixed $M\ne0$,
\[
\E\|\A_{\Om,\eps}(M)\|_2^2=
[p+\eps^2(1-p)]\fro{M}^2.
\]
For $\delta\in(0,1)$, put $u=\log(1/\delta)$ and
\[
t_\delta(M)=\sqrt{\frac{2p(1-p)u}{s_\infty(M)}}+
\frac{2u}{3s_\infty(M)},\qquad
s_\infty(M)=\frac{\fro{M}^2}{\maxn{M}^2}.
\]
With probability at least $1-\delta$,
\begin{equation}\label{eq:prob-rsc}
\|\A_{\Om,\eps}(M)\|_2^2\ge
\underbrace{\big[\eps^2+(1-\eps^2)\max\{0,p-t_\delta(M)\}\big]}_{\alpha_{p,\eps}^2(\delta,M)}\fro{M}^2.
\end{equation}
Equation~\eqref{eq:prob-rsc} describes the energy of a fixed matrix; Section~3 develops the uniform shell estimates used for recovery.
\end{lemma}
The global RSC constant $\eps^2$ is tight in general. Independent observation can improve the energy lower bound for a fixed sufficiently spread matrix, as \eqref{eq:prob-rsc} shows. For the perturbation $\epsilon=5\times10^{-5}$ used in our experiments, the deterministic constants are $\alpha_s=2.5\times10^{-9}$ and $L_s=1$, giving condition ratio $4\times10^8$ and normalized RIP constant approximately $0.999999995$. Thus the perturbation provides positive curvature in the matrix variable, while the factorized objective can retain spurious solutions. This motivates applying a lifted framework that can handle large RIP constants. The proof of Lemma~\ref{lem:epsilonMC-rip-general} appears in Appendix~\ref{app:proofs}.

\section{Quality of the Surrogate $\epsilon$-MC Problem}

Now we have established that $\epsilon$-MC problems have valid RSC constants, but another major question to ask is that \emph{whether the surrogate problem solves the original one}? Generally speaking this statement is false because $\|\mathcal{A}_{\Omega,\epsilon}(M^*)-b\|^2_2 \neq 0$; however, if we impose some simple technical constraints, then it's possible to show that the new global optimum of \eqref{eq:eps_mc} stays $\mathcal{O}(\epsilon)$ close to $M^*$. Please find all proofs and supporting lemmas in Appendix~\ref{app:proofs}, and we start with the notion of identifiability.

\subsection{Identifiability of Matrix Completion Problems}

For the unperturbed matrix-completion problem, define the feasible completion set
\begin{equation}\label{eq:completion_set}
\Sset_0(\Om,\Mstar;r)
:=\{M\succeq0:\ \rank(M)\le r,\ \A_{\Om}(M)=\A_{\Om}(\Mstar)\}.
\end{equation}
Unique completability means $\Sset_0(\Om,\Mstar;r)=\{\Mstar\}$. To quantify separation from $\Mstar$, set $D:=\sqrt5\fro{\Mstar}$ and define, for $0<\rho\le D$,
\begin{equation}\label{eq:rho_ident}
\mu_{\iid,D}(\rho):=
\inf_{\substack{M\succeq0,\ \rank(M)\le r\\
\rho\le\fro{M-\Mstar}\le D}}
\frac{\fro{\A_{\Om}(M-\Mstar)}^2}{\fro{M-\Mstar}^2}.
\end{equation}
The PSD, rank, and distance constraints define a compact annulus, on which the ratio is continuous. The upper radius follows from global optimality, rather than an imposed search constraint: writing $b_\perp:=\fro{\A_{\Omc}\Mstar}$, comparison of any global minimizer $\Mdag$ with $\Mstar$ gives
\[
\fro{\A_{\Om}(\Mdag-\Mstar)}^2+
\eps^2\fro{\A_{\Omc}\Mdag}^2\le\eps^2b_\perp^2.
\]
Thus mask orthogonality and the triangle inequality yield
\[
\fro{\Mdag-\Mstar}^2\le(\eps^2+4)b_\perp^2
\le5\fro{\Mstar}^2=D^2,
\]
using $0<\eps\le1$. Here $b_\perp$ is a scalar, distinct from the observation vector $b$.

\begin{lemma}[Bounded identifiability]\label{lem:uniq_ident_equiv}
$\Mstar$ is the unique member of $\Sset_0(\Om,\Mstar;r)$ within distance $D$ if and only if $\mu_{\iid,D}(\rho)>0$ for every $0<\rho\le D$. In particular, global unique completability is sufficient.
\end{lemma}
Theorem~2 of \citet{pimentel2016characterization} gives explicit combinatorial conditions on the observation pattern under which completion is unique for almost every rank-$r$ matrix. The implication for $\epsilon$-MC is as follows.

\begin{lemma}[Small $\epsilon$ guarantees recovery under identifiability]\label{lem:transfer_shell}
Fix $0<\rho\le D$. If
\begin{equation}\label{eq:transfer_cond}
\mu_{\iid,D}(\rho)\rho^2>\eps^2\fro{\A_{\Omc}\Mstar}^2,
\end{equation}
then every global minimizer of \eqref{eq:eps_mc} satisfies $\fro{\Mdag-\Mstar}<\rho$.
\end{lemma}
If $b_\perp=0$, the level-set bound already gives $\Mdag=\Mstar$; if $\rho>D$, localization is automatic. Lemma~\ref{lem:mu_id_monotone_explicit} establishes monotonicity and a lower bound for the bounded modulus under its stated sharpness assumption. To control the missing energy and obtain explicit error bounds, we next impose more structure on $\Omega$.

\subsection{Beyond Identifiability Under Bernoulli Sampling of Entries}

Arguably, the most common instances of MC problems is one where each entry is observed independently with probability $p$. In other words, each entry of the mask $\mathcal{A}_\Om$ is a Bernoulli random variable. Under this classic setting, we can significantly strengthen the bound between $M^*$ and the global solution $\widetilde M$ of \eqref{eq:eps_mc} using a probabilistic argument. Specifically, for $\rho>0$, let
\begin{equation}\label{eq:rho_shell}
\cH(\rho):=\{H=M-\Mstar:\ M\succeq0,\ \rank(M)\le r,\ \fro{H}\le \rho\}.
\end{equation}
Lemma~\ref{lem:transfer_shell} says if $\mu_{\iid,D}(\rho)\rho^2$ dominates the perturbation bias, then global optimality of \eqref{eq:eps_mc} forces the actual error $\Delta=\Mdag-\Mstar$ into $\cH(\rho)$. A uniform Bernstein argument on this fixed shell then gives the main Bernoulli guarantee. 

\begin{theorem}[Main result for quality of $\widetilde M$]\label{thm:main_bernoulli}
Assume $\A_{\Om}\sim \Ber(p)$ entrywise, let $\Xdag$ be a global minimizer of \eqref{eq:eps_mc}, and set $\Mdag=\Xdag{\Xdag}^{T}$. Fix $0<\rho\le D$ and $\chi\in(0,\min\{p,1-p\})$. Assume $\epsilon$ is chosen with \footnote{The bounded-modulus gap is an additional hypothesis; the radius lies in $(0,D]$.}:
\begin{equation}\label{eq:main_mu_gap}
\mu_{\iid,D}(\rho)\rho^2>\eps^2(1-p+\chi)\fro{\Mstar}^2.
\end{equation} 
then
\begin{equation}\label{eq:main_bernoulli_bound}
\fro{\Mdag-\Mstar}^2
\le
\frac{1-p+\chi}{p-\chi}\,\eps^2\fro{\Mstar}^2
\end{equation}
holds with probability at least $1-\delta_{\text{bernoulli}}$, where $\delta_{\text{bernoulli}} =$
{\small
\begin{equation}\label{eq:main_bernoulli_prob}
\begin{aligned}
2\exp\!\left[-c\min\!\left\{\frac{\chi^2 s_4(\Mstar)}{p(1-p)},\ \chi s_{\infty}(\Mstar)\right\}\right]
+2\exp\!\left[(4n+2)r\log\!\left(\frac{36}{\chi}\right)
-c\min\!\left\{\frac{\chi^2 s_{4,\rho}}{p(1-p)},\ \chi s_{\infty,\rho}\right\} \right].
\end{aligned}
\end{equation}
}
for some universal constant $c$, where $s_4(A):=\fro{A}^4/\four{A}^4$, $s_{\infty}(A):=\fro{A}^2/\maxn{A}^2$ for any non-zero $A \in \RR^{n \times n}$ and $s_{4,\rho} \coloneq \inf_{H\in\cH(\rho)\setminus\{0\}} s_4(H)$, $s_{\infty,\rho} \coloneq \inf_{H\in\cH(\rho)\setminus\{0\}} s_\infty(H)$.
\end{theorem}

Although the theorem has no explicit hard cut-off for $p$, decreasing $p$ worsens both the error bound and its success probability, which can become vacuous. The bounded-modulus gap and the shell parameters $s_{4,\rho},s_{\infty,\rho}$ are substantive assumptions; Bernoulli sampling alone does not guarantee them. If $H$ has $s$ nonzero entries of equal magnitude, then $s_4(H)=s_\infty(H)=s$, illustrating their dependence on how the error is spread across entries. Incoherence of $\Mstar$ and a sufficiently small radius allow explicit lower bounds: Proposition~\ref{prop:gnmr_shell_radius} gives $s_{4,\rho},s_{\infty,\rho}\ge n/(16\mu_0r)$. The probability bound remains conservative.

\subsection{How traditional ``goodness'' in MC improves our results}\label{sec:local_rip_mc}

Theorem~\ref{thm:main_bernoulli} gives an $\mathcal O(\eps)$ surrogate bound under its localization and shell regularity conditions. However, the probability \eqref{eq:main_bernoulli_prob} with which it holds is quite conservative to respect its generality. Nevertheless, if standard assumptions in the literature are invoked, especially an order-of-magnitude lower-bound on $p$ and that the ground truth matrix $M^*$ satisfies the incoherence condition as per Definition~\ref{def:incoherence}, then our results can be further sharpened under the same framework.

The most recent recovery result \citep{zilber2022gnmr} under the BM factorized objective \eqref{eq:main_mc} shows that if $\Mstar$ is rank $r$, $\mu_0$-incoherent, and has condition number $\kappa$, then under
\begin{equation}\label{eq:gnmr_sample}
np \ge \mathcal{O} \left( \frac{1}{\eta^2}\,\mu_0 r \max\!\left\{\log n,\ \frac{\mu_0 r\kappa^2}{\eta^2}\right\} \right),
\end{equation}
for some $\eta < 1$, the sampling operator is uniformly well conditioned on a local incoherent neighborhood of $\Mstar$. We show in Lemma~\ref{lem:gnmr_shell_rip} that if we further constrain $\rho < \rho_{\mathrm{GNMR}}$, where
\begin{equation}\label{eq:rho_gnmr_main}
\rho_{\mathrm{GNMR}}
:=
\min\!\left\{
\frac{\eta\,\sigma_r(\Mstar)}{c_e},
\frac{\sigma_r(\Mstar)}{2},
\frac{c_{\mu}\sqrt{\mu_0 r/n}}{1+c_{\mu}\sqrt{\mu_0 r/n}}\,\sigma_r(\Mstar)
\right\},
\qquad c_{\mu}:=\frac{\sqrt{3/2}-1}{\sqrt2}.
\end{equation}
and assume that \eqref{eq:gnmr_sample} holds, then lower-bound to $\mathcal{A}_\Omega (\widetilde M - M^*)$ can be improved to $\mathcal{O}(1-\eta)$ with much higher probability. Combine this with Theorem~\ref{thm:main_bernoulli} yields the main result of this subsection:

\begin{theorem}[Main theorem under GNMR local RIP]\label{thm:main_gnmr}
Fix $\eta\in(0,1)$ and $\chi\in(0,1-p)$. Assume: (i) $\Mstar$ is rank $r$, $\mu_0$-incoherent, and has condition number $\kappa$; (ii) the sample size condition \eqref{eq:gnmr_sample} holds; (iii) $0<\rho\le \rho_{\mathrm{GNMR}}\le\sigma_r(\Mstar)/2<D$; (iv) $\mu_{\iid,D}(\rho)\rho^2>\eps^2(1-p+\chi)\fro{\Mstar}^2$. Then every global minimizer of \eqref{eq:eps_mc} satisfies
\begin{equation}\label{eq:main_gnmr_bound}
\fro{\Mdag-\Mstar}^2
\le
\frac{1-p+\chi}{p(1-\eta)^2}\,\eps^2\fro{\Mstar}^2
\end{equation}
with probability at least $1-\delta_{\text{GNMR}}$ with
\begin{equation}\label{eq:main_gnmr_prob}
\delta_{\text{GNMR}} = \frac{3}{n^3}
+
2\exp\!\left[-c\min\!\left\{\frac{\chi^2 s_4(\Mstar)}{p(1-p)},\ \chi s_{\infty}(\Mstar)\right\}\right].
\end{equation}
\end{theorem}

This result complements the local guarantees of \citet{zilber2022gnmr}: it bounds the accuracy of a global surrogate minimizer once the bounded-modulus gap places it in the local-RIP neighborhood. Here the perturbation controls the accuracy of a global surrogate solution, while the over-parametrization in the next section addresses how to reach an accurate solution.

\section{The Computational Solution via Over-parametrization}
Now that we are able to reformulate the original MC problem \eqref{eq:main_mc} into the new $\epsilon$-MC problem \eqref{eq:eps_mc}, it presents us with a new challenge. Although now \eqref{eq:eps_mc} admits valid RSC/RSS constants, this is nevertheless still a difficult matrix sensing problem to solve due to its small RSC constant (or large RIP constant). Thus, it is important that we apply an over-parametrized framework to deal with it in order to compensate for the poor geometric uniformity.

To this end, we employ the lifted tensor framework proposed in \cite{ma2023over}, since it has the ability to deal with arbitrarily small RSC constants. However, in their original work, measurements were assumed to be clean, and this is incompatible with our framework. Thus, in this section we generalize the original results in \cite{ma2023over} and \cite{ma2024algorithmic} to demonstrate how the inclusion of noise could affect guarantees when using a higher-order tensor parametrization.

First of all, we present the lifted tensor problem when our observations are corrupted by some random noise $\tilde w \in \RR^m$,
\begin{equation}\label{eq:lifted_main_highr}
	\min_{\mb{w} \in \RR^{nr \circ l}} \quad \| \langle \mb{A}^{\otimes l}, \langle \mb{P}(\mb{w}), \mb{P}(\mb{w}) \rangle_{2*[l]} \rangle - \tilde b^{\otimes l} \|^2_F
\end{equation}
where $\tilde b = \mathcal{A}(M^*) + \tilde w$, and $\mathbf{A} \in \RR^{m \times n \times n}$ is a three-way tensor which can be seen as a concatenation of all sensing matrices $\{A_i\}_{i=1}^m$, and $\mb{w} \in \RR^{nr \circ l}$ is the tensor decision variable used to increase the parametrization of $X \in \RR^{n \times r}$. Here, $^{\otimes l}$ simply denotes $l$ times of repeated tensor outer product, and $\mb{P}$ is just another constant permutation tensor used for correct multiplication. The gist of this paper is not on the tensor formulation, thus many details are deferred to Appendix~\ref{app:noisy_lifted}, and interested readers can learn more about general tensor knowledge and problem details there. For convenience's sake, we define $f^l(\cdot): \RR^{n \circ 2l} \mapsto \RR$ and $h^l(\cdot): \RR^{[n \times r] \circ l} \mapsto \RR$ as $f^l(\mb{M}) \coloneqq \| \langle \mb{A}^{\otimes l}, \mb{M} \rangle - \tilde b^{\otimes l} \|^2_F$ and $h^l(\mb{w}) = f^l(\langle \mb{w}, \mb{w} \rangle_{2*[l]})$, with $\nabla f^l(\cdot) = \nabla_\mb{M} f^l(\cdot)$ and $\nabla h^l(\cdot) = \nabla_\mb{w} h^l(\cdot)$.

In the original works, it was proven that the lifted formulation \eqref{eq:lifted_main_highr} is able to convert spurious second-order points in \eqref{eq:main_ms} to strict saddle points (critical points with local escape direction) via its drastic over-parametrization if this spurious solution is somehow far away from the ground truth $M^*$. Therefore, we have established a new result catering to our noisy formulation~\eqref{eq:eps_mc} below to show that even in the noisy situation, similar results still hold:
\begin{theorem}\label{thm:socp_lifted_noisy}
	Consider an arbitrary second-order point $\hat X \in \RR^{n \times r}$ of the factorized matrix sensing objective in the form of \eqref{eq:main_ms} where its observations $b$ could be potentially corrupted by some random noise $\tilde w \in \RR^m$ (i.e. $b = \tilde b$). Assuming that the linear operator $\mathcal{A}(\cdot)$ in \eqref{eq:main_ms} satisfies the RSC and RSS conditions with constants $\alpha_s, L_s$ respectively. Then $\mb{\hat w} = \vecc(\hat X)^{\otimes l}$ is a strict saddle of \eqref{eq:lifted_main_highr} with a rank-1 symmetric escape direction if
	\begin{equation}\label{eq:thm_socp_distance_lower_highr}
		\|M^* - \hat X \hat X^\top \|^2_F \geq \frac{L_s}{\alpha_s} \lambda_r(\hat X \hat X^\top) \tr(M^*) +  \frac{\|\tilde w\|^2_2}{\alpha_s}
	\end{equation}
	with an odd $l$ satisfying
	\begin{equation}\label{eq:thm_socp_lcond_highr}
		l > \frac{1}{1-\log_2(2 \beta)}, \quad \beta \coloneqq \frac{ L_s \tr(M^*) \lambda_r(\hat X \hat X^\top)}{\alpha_s  \|M^* - \hat X \hat X^\top \|^2_F - \|\tilde w \|^2_2 }. 
	\end{equation}
\end{theorem}

The proof of this theorem is located in Appendix~\ref{app:proofs}. The theorem highlights how the conversion radius from spurious solutions to strict saddles is influenced by the norm of the noise $\tilde w$. Setting $\tilde w = 0$ allows this theorem to coincide with Theorem 4 from \cite{ma2024algorithmic}. While the above result holds for all noisy matrix sensing problems, if we plug-in the configurations of our $\epsilon$-MC formulation with $\tilde w = -\epsilon \A_{\Omc}\Mstar$, we get an immediate corollary:
\begin{corollary}
\label{cor:general}
Consider the matrix completion problem of completing a $n \times n$, rank-$r$ matrix $M^*$, where $\Omega \subseteq [n] \times [n]$ denotes the set of observed entries and $\Omc$ denotes the unobserved entries. If we solve its $\epsilon$-MC surrogate as per \eqref{eq:eps_mc}, then any rank-1 local minima $\mb{\hat w} = \vecc(\hat X)^{\otimes l}$ of \eqref{eq:lifted_main_highr} with $\mathcal{A} = \mathcal{A}_{\Omega,\epsilon}$ satisfies:
\begin{equation}
	\frac{ \|M^* - \hat X \hat X^\top \|_F}{ \lambda_r(\hat X)} <  \frac{\sqrt{\tr(M^*)}}{\alpha_{p,\epsilon} } +  \frac{\fro{\A_{\Omc}\Mstar}}{\lambda_r(\hat X) }
\end{equation}
Here $\alpha_{p,\epsilon}\ge\epsilon$ denotes an admissible norm-scale lower bound for the comparisons required by Theorem~\ref{thm:socp_lifted_noisy}, with quadratic RSC constant $\alpha_s=\alpha_{p,\epsilon}^2$. The deterministic choice $\alpha_{p,\epsilon}=\epsilon$ is always valid; a sharper uniform bound requires separate justification beyond the fixed-matrix estimate \eqref{eq:prob-rsc}.
\end{corollary}
Although this results gives a direct upper-bound to the normalized distance between the local minimum $\hat X$ and $M^*$, the term $\fro{\A_{\Omc}\Mstar}/\lambda_r(\hat X)$ is unknown and could be potentially very large. To address this issue, we present the following main result by assuming more specific structures:
\begin{theorem}[Main Result]
	\label{thm:main_result}
Consider the setting described in Corollary~\ref{cor:general} where $\Omega$ is obtained with i.i.d. observation of probability $p$. We choose an $\epsilon$ that satisfies \eqref{eq:main_mu_gap} for some $0<\rho\le D$. Then applying the lifted formulation to the surrogate $\epsilon$-MC problem yields the following results:

For any rank-1 local minima $\mb{\hat w} = \vecc(\hat X)^{\otimes l}$ of \eqref{eq:lifted_main_highr}, if the conditions of \eqref{eq:thm_socp_lcond_highr} are met with $L_s = 1, \alpha_s = \alpha_{p,\epsilon}^2$ and $\tilde w = 0$, then it is guaranteed to satisfy that:
\begin{equation}\label{eq:thm_bernoulli_gaurantee}
   \|M^* - \hat X \hat X^\top \|_F <  \frac{\lambda_r(\hat X)}{\alpha_{p,\epsilon} }  \sqrt{\tr(M^*)}+ \eps \sqrt{\frac{1-p+\chi}{p-\chi}} \fro{\Mstar}
\end{equation}
for any $\chi\in(0,\min\{p,1-p\})$ with probability at least $1-\delta_{\text{bernoulli}}$ given in \eqref{eq:main_bernoulli_prob}. Additionally, if the ground truth matrix $M^*$ is also $\mu_0$-incoherent with some constant $\mu_0$, the bound can be further strengthened to:
\begin{equation}\label{eq:thm_gnmr_gaurantee}
    \|M^* - \hat X \hat X^\top \|_F < \frac{\lambda_r(\hat X)}{\alpha_{p,\epsilon} } \sqrt{\tr(M^*)} + \eps \frac{\sqrt{(1-p+\chi)/p} \fro{\Mstar}}{1-\eta}
\end{equation}
by imposing the requirements on $p$ and $\epsilon$ in Theorem~\ref{thm:main_gnmr}, where the probability also improves to $\delta \geq \delta_{\text{GNMR}}$ in \eqref{eq:main_gnmr_prob}.
\end{theorem}

Theorem~\ref{thm:main_result} follows by applying the noiseless lifting comparison of \citet{ma2024algorithmic} around the surrogate minimizer $\Mdag$: global optimality makes its objective gap nonpositive, without requiring zero measurement residual. Theorems~\ref{thm:main_bernoulli} and~\ref{thm:main_gnmr} bound $\fro{\Mdag-\Mstar}$, so the triangle inequality gives \eqref{eq:thm_bernoulli_gaurantee} and \eqref{eq:thm_gnmr_gaurantee}. The noisy results in Theorem~\ref{thm:socp_lifted_noisy} and Corollary~\ref{cor:general} compare directly to $\Mstar$ and are not used here.

Although the result above holds only for rank-one critical points, its approximately rank-one extension is discussed in Theorem~\ref{thm:main_ext} of Appendix~\ref{app:noisy_lifted}. Theorem~\ref{thm:implicit_bias_gd} establishes approximately rank-one behavior for plain GD under small-initialization and finite-horizon conditions. Appendix~\ref{app:algorithmic_guarantees} gives the escape-curvature constant and a conditional perturbed-GD consequence: complete strict-saddle alternatives, derivative bounds, and trajectory containment yield $\widetilde{\mathcal O}(L_h\Delta_h/\eta_g^2)$ iterations to the protected neighborhood~\citep{jin2017escape}. These additional conditions are distinct from plain GD's implicit bias; a full GD--Tensor PCA runtime bound remains open.

%

\section{Numerical Experiments}
In this section, we evaluate our method\footnote{Code: \url{https://github.com/MOFA-LAB/mc_as_noisyms}. Experiments run on an M1 Max MacBook Pro.} against both classic and state-of-the-art solvers on two families of matrix completion (MC) instances. In addition to the uniformly and independently sampled instances as in Theorem~\ref{thm:main_bernoulli}, we also include the hard instance from \cite{yalccin2022factorization}, which is known to admit exponentially many spurious solutions:
\begin{equation}\label{eq:structured_prob}
    \Omega = \{ (i,i), (i,2k), (2k,i) \mid \forall i \in [n],\; k \in [ \lfloor n/2 \rfloor]\}.
\end{equation}
We compare our methods to: (i) the classic Burer--Monteiro (BM) factorization baseline for \eqref{eq:main_mc}; (ii) the SDP relaxation \citep{candes2010power}; (iii) GNMR~\citep{zilber2022gnmr}, a state-of-the-art BM method that solves a linearized least-squares subproblem each iteration (the corrected reference implementation, Algorithm~3 with $\alpha=-1$, is evaluated separately in Table~\ref{tab:direct_large}); (iv) Lin--RFM~\citep{radhakrishnan2025linear}, an SVD-free IRLS/Schatten-$p$ spectral-reweighting scheme; and (v) our lifted approach with $l=3$. Our method produces a tensor $\mb{w}_T$ after $T$ gradient steps, and Tensor PCA~\citep{ma2024algorithmic} extracts a principal component $X_T \in \RR^{n\times r}$ via $\mb{w}_T \approx \vecc(X_T)^{\otimes l}$, with $X_T$ being returned for \eqref{eq:main_mc}. We declare success if $\|X_T X_T^\top - M^*\|_F \le 0.05$ and report \emph{success rates} over 20 trials per problem, rather than average reconstruction error, because they are robust to outliers and better capture instances where $\epsilon$-MC may be an invalid surrogate. Ablation studies of $\epsilon$ and $p$ are included in Appendix~\ref{app:ablation}.

\begin{figure*}[ht] 
    \centering
    \subfigure[Problem Instance \eqref{eq:structured_prob}]{\label{fig:success1}\includegraphics[width=0.49\textwidth]{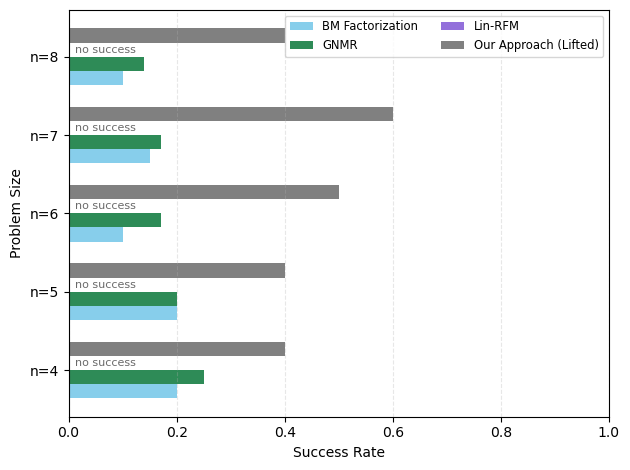}}
    \hfill
    \subfigure[Independent Entries with $p=0.15$]{\label{fig:success2}\includegraphics[width=0.49\textwidth]{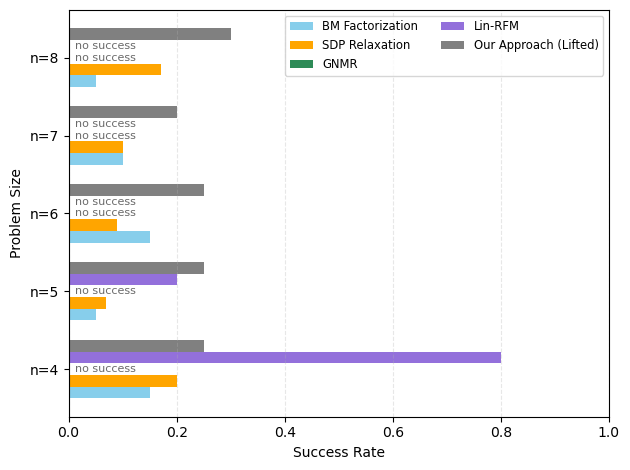}}
    
    \caption{We use $\epsilon=5\times10^{-5}$, a learning rate of $2\times10^{-2}$, and an initialization scale of $\xi=10^{-4}$ for \eqref{eq:lifted_main_highr}. SDP was not deployed for (a) since the solution is not unique in the relaxed regime. For (b), the variability of $\Omega$ necessitates running 10 distinct problem instances for each size $n$.}
    \label{fig:success}
\end{figure*}
We see from figure~\ref{fig:success} that across both instance families and sizes $n$, our lifted approach attains consistently higher success rates. The GNMR baseline in the original figure is competitive on the structured instance but degrades on the independent model, while Lin--RFM, as an IRLS/Schatten-$p$ method, shows strong performance at smaller sizes in the independent model and is less effective on the structured instance. Appendix~\ref{app:error-spatial} examines unsuccessful recoveries on the structured instance: all ten runs fit observed entries to mean-squared error below $5\times10^{-8}$, while seven have missing-entry error about $2.67$. The error maps locate these failures on unobserved entries.

To separate perturbation from lifting, we compare direct optimization of MC and $\epsilon$-MC on rank-five $n=200$ instances. Each solver uses $10$ matched problems, masks, and initializations per sampling rate, with $\epsilon=5\times10^{-5}$. Success means Frobenius error at most $0.05$; for reference GNMR (Algorithm~3, $\alpha=-1$), we evaluate $UV^\top$. Table~\ref{tab:direct_large} reports success rates (\%) and their change in percentage points.

\begin{table}[ht]
\centering
\caption{Direct recovery at $n=200$: perturbation changes only Adam's success rate.}
\label{tab:direct_large}
\renewcommand{\arraystretch}{1.15}
\setlength{\tabcolsep}{3pt}
\begin{tabular*}{\textwidth}{@{\extracolsep{\fill}}l*{3}{rrr}@{}}
\toprule
& \multicolumn{3}{c}{$p=0.02$} & \multicolumn{3}{c}{$p=0.05$} & \multicolumn{3}{c}{$p=0.08$}\\
\cmidrule(lr){2-4}\cmidrule(lr){5-7}\cmidrule(l){8-10}
Solver & MC & $\epsilon$-MC & $\Delta$ (pp) & MC & $\epsilon$-MC & $\Delta$ (pp) & MC & $\epsilon$-MC & $\Delta$ (pp)\\
\midrule
Adam & 0 & 0 & 0 & 30 & 40 & $\mathbf{+10}$ & 70 & 60 & $\mathbf{-10}$\\
GNMR (reference) & 0 & 0 & 0 & 0 & 0 & 0 & 0 & 0 & 0\\
GD (Armijo) & 0 & 0 & 0 & 90 & 90 & 0 & 100 & 100 & 0\\
L-BFGS & 0 & 0 & 0 & 90 & 90 & 0 & 100 & 100 & 0\\
\bottomrule
\end{tabular*}
\end{table}

Here $\Delta$ is the $\epsilon$-MC success rate minus the MC success rate. Only two of the $120$ paired outcomes change, both for Adam, showing that perturbation alone does not systematically resolve direct solvers' failures. Figure~\ref{fig:success} tests the combined approach, whereas Table~\ref{tab:direct_large} isolates perturbation within each solver. Together they illustrate how perturbation provides a regular sensing surrogate while lifting addresses the remaining optimization difficulty. Appendix~\ref{app:direct_controls} gives further controls.

\paragraph{Choosing $\epsilon$.}
Appendix~\ref{app:ablation} studies the structured $n=6$, $l=3$ instance with $100$ trials per perturbation. At the $0.05$ threshold, success is approximately $26\%$ at $\epsilon=5\times10^{-4}$ and $36\%$ at $5\times10^{-5}$, versus about $9\%$ for BM, but drops to $11\%$ at $10^{-5}$. This illustrates the tradeoff between perturbation bias and conditioning. The sampling-rate sweep shows the largest improvement over BM at small $p$.

\paragraph{Limitations.}
Choosing $\epsilon$ involves a tradeoff: smaller perturbations reduce surrogate bias but weaken the guaranteed curvature, making optimization harder. Explicit order-$l$ tensors require $\mathcal O((nr)^l)$ storage, limiting our implementation to small-scale problems. Simulated Oracle Direction (SOD)~\citep{shen2026simulated} uses lifting to construct escape candidates in the original factor space without storing the tensor, alleviating this cost. Appendix~\ref{app:sod} reports favorable results on larger structured instances; full adoption for noisy $\epsilon$-MC remains future work. Our guarantees concern symmetric PSD factors, a two-level mask, and rank-one or approximately rank-one lifted points. Broader settings and complete runtime and recovery guarantees, including Tensor PCA and factor extraction, remain further directions.

\section{Conclusion}

In conclusion, we shift from the classical goal of exact recovery of $M^*$ to solving a carefully designed surrogate that remains provably close to $M^*$ and whose landscape is amenable to over-parameterization. Our analysis separates surrogate accuracy from the optimization benefits of lifting, and the experiments illustrate how perturbation and over-parametrization can complement one another in matrix completion.

\clearpage
\begin{ack}
This work was supported by the National Natural Science Foundation of China (Grant No.~62506314), the Research Grants Council of Hong Kong (Grant No.~21208525), and City University of Hong Kong (Project No.~9382001).
\end{ack}

\bibliography{references.bib}


\appendix
\newpage

\section{Additional Ablation Studies}
\label{app:ablation}

In this section we complement our main experiments with ablation studies on the
user-chosen perturbation level $\epsilon$ and the sampling probability $p$ for
the structured instance in problem~\eqref{eq:structured_prob}. Throughout, we
consider the setting $n = 6$, depth $\ell = 3$, and initialization magnitude
$10^{-4}$.

\begin{figure}[ht]
    \centering
    \includegraphics[width=0.8\columnwidth]{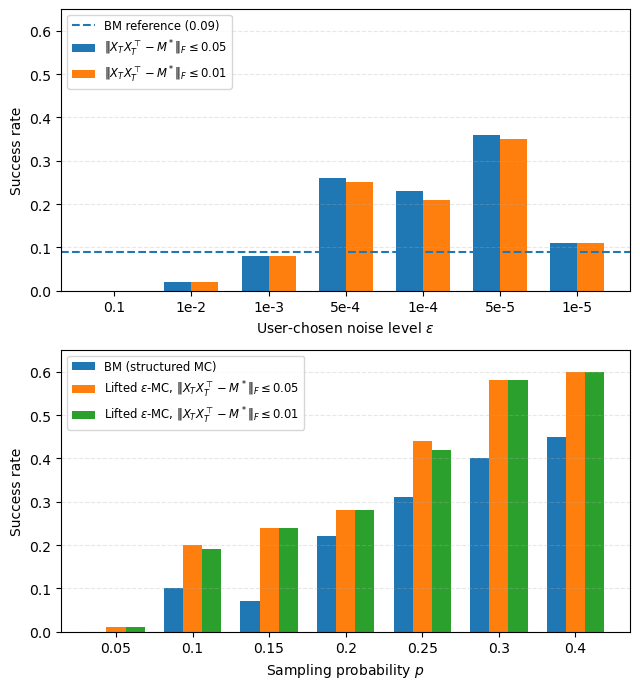}
    \caption{Ablation studies for the lifted $\epsilon$-MC formulation with $n=6$, $\ell = 3$,
    initialization magnitude $10^{-4}$. \emph{Top:} Success rates with respect to the perturbation level $\epsilon$, under two
    reconstruction criteria
    $\|X_T X_T^\top - M^\star\|_F  \le 0.05$ and
    $\|X_T X_T^\top - M^\star\|_F  \le 0.01$, together with the BM baseline
    success rate (dashed line) for the
    structured instance~\eqref{eq:structured_prob}. Success rate is polled from 100 independent trials each. \emph{Bottom:} Success rates of the independently sampled problem with respect to sampling probability $p$, comparing the BM approach and the lifted
    $\epsilon$-MC method under both reconstruction criteria. For each $p$ level, 10 random masks are generated and 20 trials each are run for a total of 200 trials.}
    \label{fig:epsilon-p-ablation}
\end{figure}

\paragraph{Ablation on $\epsilon$.}
For the first study we fix the sampling pattern as in
problem~\eqref{eq:structured_prob} and vary the perturbation level
$\epsilon \in \{0.1, 10^{-2}, 10^{-3}, 5{\times}10^{-4}, 10^{-4},
5{\times}10^{-5}, 10^{-5}\}$.
For each $\epsilon$, we record the empirical success rates under the two
reconstruction criteria
$\|X_T X_T^\top - M^\star\|_F  \le 0.05$ and
$\|X_T X_T^\top - M^\star\|_F  \le 0.01$.
As shown in the top panel of Figure~\ref{fig:epsilon-p-ablation}, the success
curves for the two thresholds closely track each other, indicating that this
problem behaves as a ``success or failure'' scenario rather than exhibiting a
gradual degradation in reconstruction quality.

These trends are consistent with the error bound in Theorem~\ref{thm:main_result}, which decomposes into a \emph{computational part}, proportional to $1/\epsilon$ (this rate is rather conservative and could be improved, please see Lemma~\ref{lem:epsilonMC-rip-general}), and an \emph{accuracy part} $e_1$ capturing the mismatch between the $\epsilon$-perturbed objective and the original matrix
completion problem. For large $\epsilon$ (e.g., $0.1$ and $10^{-2}$), the
perturbation error $e_1$ dominates and the success rate is low even if the
optimizer reaches a global minimum of the $\epsilon$-MC landscape. As $\epsilon$
decreases into an intermediate regime (roughly between $5{\times}10^{-4}$ and
$5{\times}10^{-5}$), the two terms balance and the success rate peaks. When
$\epsilon$ becomes too small (e.g., $10^{-5}$), the $1/\epsilon$ factor renders
the problem numerically more challenging, and the empirical success rate
deteriorates. In this example, the useful range spans approximately an order of magnitude; $\epsilon$ requires tuning, and this sweep does not supply a universal selection rule.

\paragraph{Ablation on sampling probability $p$.}
In the second study we fix $\epsilon = 5{\times}10^{-4}$ and vary the sampling
probability $p \in \{0.05, 0.1, 0.15, 0.2, 0.25, 0.3, 0.4\}$.
For each $p$, we again estimate the success rates over 10 randomly generated masks with 20 trials each amounting to 200 trials in total for both the BM method and our approach. The results are summarized in the bottom panel of
Figure~\ref{fig:epsilon-p-ablation}.

As $p$ increases, both methods become more successful, as expected from having
more observed entries. The lifted $\epsilon$-MC approach exhibits a clear
advantage in the low-$p$ regime, where the BM formulation struggles to recover
the independently sampled instance. At larger $p$, the BM baseline itself achieves high
success rates, which naturally reduces the margin between the two methods. At
the same time, further gains for our approach appear limited by the same
computational term in the error bound: while the accuracy part improves rapidly
with $p$, the difficulty of optimizing the lifted problem becomes the dominant
factor. 

Overall, these ablation studies illustrate how the choice of $\epsilon$
and the sampling probability $p$ jointly govern the empirical behavior of the
lifted $\epsilon$-MC formulation.

\section{Additional Details for Noisy Lifted Framework} \label{app:noisy_lifted}

\subsection{Additional Definitions}
\begin{definition}[RIP] \citep{candes2009exact} \label{def:rip}
    Given a natural number $p$, the linear map $\mathcal{A}: \mathbb{R}^{n \times n} \mapsto \mathbb{R}^{m}$ is said to satisfy $\delta_{p}$-RIP if there is a constant $\delta_{p} \in [0,1)$ such that
    \[ (1-\delta_{p})\|M \|_F^2 \leq \| \mathcal{A} (M) \|^2 \leq (1 + \delta_{p}) \|M\|_F^2 \]
    holds for matrices $M \in \mathbb{R}^{n \times n}$ satisfying $\rk(M) \leq p$.
\end{definition}

\begin{definition}[Tensor]
	As a generalization of the way vectors are used to parametrize finite-dimensional vector spaces, we use \emph{arrays} to parametrize tensors generated from product of finite-dimensional vector spaces, as per \cite{comon2008symmetric}. In particular, we define an $l$-way array as such:
	\[
		\mathbf{a} = \{a_{i_1i_2\dots i_l} | 1 \leq i_k \leq n_k, 1 \leq k \leq l \} \in \RR^{n_1 \times \dots \times n_l}
	\]
	Note that in this paper tensors and arrays can be regarded as synonymous since there exists an isomorphism between them. Moreover, if $n_1 = \dots = n_l$, then we call this tensor(array) an $l$-order(way), $n$-dimensional tensor. For the convenience of tensor representation, we use the notation $\RR^{n \circ l}$ with $n \circ l \coloneqq n \times \dots \times n$. In this work, tensors are denoted with bold variables, and other fonts are reserved for matrices, vectors, and scalars unless specified otherwise.
\end{definition}

\begin{definition}[Symmetric Tensor]
	Similar to the definition of symmetric matrices, for an order-$l$ tensor $\mb{a}$ with the same dimensions (i.e., $n_1 = \dots = n_l$), also called a cubic tensor, it is said that the tensor is symmetric if  its entries are invariance under any permutation of their indices:
	\[
		a_{i_{\sigma(1)} \cdots i_{\sigma(l)}} = a_{i_1 \cdots i_l} \quad \forall \sigma, \quad i_1, \dots, i_l \in \{1,\dots,n\}
	\]
	where $\sigma \in \mathcal{G}_l$ denotes a specific permutation and $\mathcal{G}_l$ is the symmetric group of permutations on $\{1, \dots, l\}$. We denote the set of symmetric tensors as $\mathrm{S}^l(\RR^n)$.
\end{definition}

\begin{definition}[Rank of Tensors]
	The rank of a cubic tensor $\mb{a} \in \RR^{n \circ l}$ is defined as
	\[
		\rk(\mb{a}) = \min\{r | \mb{a} = \sum_{i=1}^r u_i \otimes v_i \otimes \cdots \otimes w_i \}
	\]
	for some vector $u_i, \dots, w_i \in \RR^n$. Furthermore, according to \cite{kolda2015numerical}, if $\mb{a}$ is a symmetric tensor, then it can be decomposed as:
	\[
		\mb{a} = \sum_{i=1}^r \lambda_i u_i \otimes \dots \otimes u_i \coloneqq \sum_{i=1}^r \lambda_i u_i^{\otimes l}
	\]
	and the rank is conveniently defined as the number of nonzero $\lambda_i$'s, which is very similar to the rank of symmetric matrices indeed. The most important concept in our paper is rank-1 tensors, and for any tensor $\mb{a}$, a necessary and sufficient condition for it to be rank-1 is that 
	\[
		\mb{a} = u^{\otimes l}
	\]
	for some $u \in \RR^n$.
	
\end{definition}

\begin{definition}[Tensor Multiplication]
	Outer product is an operation carried out on a pair  of tensors, denoted as $\otimes$. The outer product of 2 tensors $\mathbf{a}$ and $\mathbf{b}$, respectively of orders $l$ and $p$, is a tensor of order $l+p$, denoted as $\mb{c} = \mb{a} \otimes \mb{b}$ such that:
	\[
		c_{i_1\dots i_l j_1 \dots j_p} = a_{i_1\dots i_l} b_{j_1 \dots j_p}
	\]
	When the 2 tensors are of the same dimension, this product is such that $\otimes: \RR^{n \circ l} \times \RR^{n \circ p} \mapsto \RR^{n \circ (l+p)}$. Henceforth, we use the shorthand notation
	\[
		 \underbrace{a \otimes \dots \otimes a}_{l \ \text{times}}\coloneqq  a^{\otimes l}
	\]
	We also define an inner product of two tensors. The mode-$q$ inner product between the 2 aforementioned tensors having the same $q$-th dimension is denoted as $\langle \mb{a}, \mb{b} \rangle_q$. Without loss of generality, assume that $q = 1$ and 
	\[
		\left[ \langle \mb{a}, \mb{b} \rangle_q \right]_{i_2\dots i_l j_2 \dots j_p} = \sum_{\alpha=1}^{n_q} a_{\alpha i_2\dots i_l} b _{\alpha j_2 \dots j_p}
	\]
	Note that when we write $\langle \cdot, \cdot \rangle_q$, we count the $q$-th dimension of the first entry. Indeed, this definition of inner product can also be trivially extended to multi-mode inner products by just summing over all modes, denoted as $\langle \mb{a}, \mb{b} \rangle_{q, \dots, s}$.
\end{definition}

\begin{lemma}[Section 10.2 \cite{petersen2008matrix}]
For four arbitrary matrices $A,B,C,D$ of compatible dimensions, it holds that

\begin{equation}
	\langle A \otimes B, C\otimes D \rangle_{2,4} = AC \otimes BD
\end{equation}
	
\end{lemma} \label{lem:kron_iden}

\begin{definition}[Variational Eigenvalue of Tensors \cite{ma2024algorithmic}]\label{def:v_eigenvalues}
	For a given tensor $\mb{w} \in \RR^{n \circ l}$, we define its $k^{th}$ variational eigenvalue (v-Eigenvalue) $\lambda_k^v(\mb{w})$ as
	\[
		\lambda_k^v(\mb{w}) \coloneqq \max_{\substack{S \\ \dim(S)=k}} \min_{\mb{u} \in S} \frac{|\langle \mb{w}, \mb{u} \rangle |}{\|\mb{u}\|^2_F}, \quad k \in [n]
	\]
	where $S$ is a subspace of $\RR^{n \circ l}$ that is spanned by a set of orthogonal, symmetric, rank-1 tensors. Its dimension denotes the number of orthogonal tensors that span this space. 
\end{definition}

\subsection{Formulation Details}
As noted in our main formulation \eqref{eq:lifted_main_highr}, the decision variable $\mb{w}$ is a tensor of dimension $nr \times \dots \times nr$, since it serves as a repeated outer product of $\vecc(X)$ with $X \in \RR^{n \times r}$ being our original decision variable in \eqref{eq:main_ms} (here we assume $\rsearch = r$). The permutation $\mb{P}$ is needed in order to convert $\mb{w} \in \RR^{nr \circ l}$ back to $\RR^{[n \times r] \circ l}$ in order to do meaningful inner products. $\mb{P} \in \RR^{n \times r \times nr}$ is defined as 
\[
	\langle \mb{P}, \vecc(X) \rangle_3 = X \quad \forall X \in \RR^{n \times r}, n, r \in \mathbb{Z^+}
\]
Such $\mb{P}$ can be easily constructed via filling appropriate scalar "1"s in the tensor. Via Lemma~\ref{lem:kron_iden}, we also know that 
\begin{equation}\label{eq:highr_vectorization}
	 \langle \mb{P}^{\otimes l}, \vecc(X)^{\otimes l} \rangle_{3*[l]}= (\langle \mb{P}, \vecc(X) \rangle_3)^{\otimes l}= X^{\otimes l}
\end{equation}
Notationally, we abbreviate $\langle \mb{P}^{\otimes l}, \mb{w} \rangle_{3*[l]}$ as $\mb{P}(\mb{w})$ for enhanced readability for an arbitrary tensor $\mb{w}$ with dimension greater or equal to $2$. 

Since we also make extensive use of first and second order critical points of \eqref{eq:lifted_main_highr}, we present them here for accessibility:

\begin{lemma}\label{lem:cp_lifted_highr}
	The tensor $ \mb{\hat w} \in \RR^{nr \circ l}$ is an SOP of \eqref{eq:lifted_main_highr} if and only if
	\begin{subequations}
		\begin{gather}
			\langle \nabla f^l( \langle \mb{P}(\mb{\hat w}), \mb{P}(\mb{\hat w}) \rangle_{2*[l]}) , \mb{P}(\mb{\hat w})\rangle_{2*[l]} = 0, \label{eq:focp_lifted_highr}\\
		\begin{split}
			&2 \langle \nabla f^l(\langle \mb{P}(\mb{\hat w}), \mb{P}(\mb{\hat w}) \rangle_{2*[l]}), \langle \mb{P}(\Delta), \mb{P}(\Delta) \rangle_{2*[l]} + \\
			\| \langle \mb{A}^{\otimes l} , \langle \mb{P}(\mb{\hat w}),& \mb{P}(\Delta) \rangle_{2*[l]} + \langle \mb{P}(\Delta), \mb{P}(\mb{\hat w}) \rangle_{2*[l]} \rangle \|_F^2 \geq 0 \quad \forall \Delta \in \RR^{nr \circ l}
		\end{split}
		\label{eq:socp_lifted_highr}
		\end{gather}
	\end{subequations}
	with \eqref{eq:socp_lifted_highr} being a necessary and sufficient condition for $\mb{\hat w}$ to be a FOP and $\nabla f_w^l( \mb{M})$ is defined as 
	\begin{equation}
	\nabla f_w^l( \mb{M}) =   \langle (\mb{A}^{\otimes l})^*\mb{A}^{\otimes l}, \mb{M} \rangle - \left[ \langle \mb{A}^* \mb{A}, M^* \rangle + \langle \mb{A}^*, \tilde w \rangle \right]^{\otimes l}
\end{equation}
\end{lemma}
The proof to this lemma is highly technical and can be obtained by slightly changing the proof to Lemma 7 in \cite{ma2024algorithmic} by changing $b = \mathcal{A}(M^*)$ to $\tilde b = \mathcal{A}(M^*) + \tilde w$ defined above.

Here we present the full theorem regarding implicit bias in \eqref{eq:lifted_main_highr}:

\begin{theorem}\label{thm:implicit_bias_gd}
	Consider a finite-horizon gradient descent trajectory $\{ \mb{w}_t \}_{t \in [T]}$ of \eqref{eq:lifted_main_highr} with $\mb{w}_{t+1} = \mb{w}_{t} - \eta \nabla h^l(\mb{w}_{t})$ starting from the initialization $\mb{w}_0 = \xi x_0^{\otimes l}$ with $\xi \in \RR$ denoting the scale of the initialization, $\eta$ representing the step-size and $x_0 \in \RR^{nr}$ being an arbitrary vector with $\|x_0\|^2_2 = 1$. Then there exists $t(\kappa,l) \geq 1$ and $\kappa < 1$ such that
		\begin{equation}\label{eq:w_eig_kappa}
			\frac{\lambda^v_{2}(\mb{w}_t)}{\lambda^v_1(\mb{w}_t)} \leq \kappa, \qquad \forall t \in [t(\kappa,l), T]
		\end{equation}
		if the initialization scale $\xi$ is sufficiently small, where $t(\kappa, l)$ is expressed as
		\begin{equation}\label{eq:t_kappa_l}
			t(\kappa, l) = \left \lceil \ln\left( \frac{\|x_0\|^l_2}{\kappa |v_1^\top x_0|^l}\right) \ln\left( \frac{ 1+\eta \sigma_1^l(U)}{1+\eta \sigma_2^l(U)} \right)^{-1} \right \rceil
		\end{equation}
		where $\sigma_{1}(U)$ and $\sigma_{2}(U)$ denote the first and second singular values of $U$ and $v_1,v_2$ are the associated singular vectors, with
		\begin{equation}
			U = \langle \mb{A}_r, \tilde b \rangle_{1} \in \RR^{nr \times nr}, \quad  \mb{A}_r \coloneqq I_r \oslash_{2,3} \mb{A}
		\end{equation}
\end{theorem}

Next, we present a technical extension of Theorem~\ref{thm:socp_lifted_noisy} and Theorem~\ref{thm:implicit_bias_gd}, showing how gradient descent initialized with small scale can help ensure that second-order points of lifted version of \eqref{eq:eps_mc} remain very close to $M^*$ along the optimization trajectory
\begin{theorem}\label{thm:main_ext}
Consider a generic matrix completion problem under the same premise as given in Theorem~\ref{thm:socp_lifted_noisy}. Assume that the symmetric tensor $\mb{\hat w} \in \RR^{nr \circ l}$ is a second-order point (local minima) of \eqref{eq:lifted_main_highr} that is $\kappa$-rank-1 with $\kappa \leq \mathcal{O}(1/\|M^*\|^2_F)$. This can be achieved by initializing the vanilla gradient algorithm at $\mb{w}_0 = \xi x_0^{\otimes l}$ with a sufficiently small $\xi > 0 \in \RR$. Then after iterations $t(\kappa, l)$ given in \eqref{eq:t_kappa_l}, Theorem~\ref{thm:implicit_bias_gd} ensures that all tensors along the trajectory will become $\kappa$-rank-1. 

If $\mb{\hat w}$'s major spectral decomposition is given as $\mb{\hat w} = \lambda_S \hat x^{\otimes l} + \mb{\hat w^\dagger}$ with $\hat x \in \RR^{nr}$ being a FOP of \eqref{eq:eps_mc} (ensured by Proposition 2 in \cite{ma2024algorithmic}), we know that
\begin{equation}
    \|M^* - \hat X \hat X^\top \|_F < \frac{1}{\epsilon} \lambda_r(\hat X) \sqrt{\tr(M^*)} + \mathcal{O}(\sqrt{r} \kappa^{1/2l}) + e_{1}
\end{equation}
holds with probability at least $q$, under the condition that $l$ is odd and meets the requirement: 
\begin{equation}
    l > \frac{1}{1-\log_2(2 \beta)}, \quad \beta \coloneqq \frac{ \tr(M^*) \lambda_r(\hat X \hat X^\top)}{\epsilon^2\left(  \|M^* - \hat X \hat X^\top \|^2_F- \mathcal{O}(r \kappa^{1/l}) - e_2 \right)}. 
\end{equation}
where $e_1$, $e_2$, and $q$ are identical to those given in Theorem~\ref{thm:socp_lifted_noisy} depending on different MC instances.
\end{theorem}
The proof of this Theorem is omitted because it directly follows from Theorem~\ref{thm:socp_lifted_noisy} , Theorem~\ref{thm:implicit_bias_gd}, and Theorem 2 in \cite{ma2024algorithmic}.

\section{Proof Details} \label{app:proofs}

\subsection{Auxiliary Results}

We collect the auxiliary ingredients used in Section~3.

\begin{lemma}[Basic inequality]\label{lem:basicineq}
Let $\Mdag$ be any global minimizer of \eqref{eq:eps_mc}. Then
\begin{equation}\label{eq:basicineq}
\frac12\fro{\A_{\Om}\Delta}^2+\frac{\eps^2}{2}\fro{\A_{\Omc}(\Delta+\Mstar)}^2
\le \frac{\eps^2}{2}\fro{\A_{\Omc}\Mstar}^2.
\end{equation}
Consequently,
\begin{align}
\fro{\A_{\Om}\Delta}
&\le \eps\fro{\A_{\Omc}\Mstar},
\label{eq:obs_bound}
\\
\fro{\A_{\Omc}\Mdag}
&\le \fro{\A_{\Omc}\Mstar}.
\label{eq:unobs_Mdag}
\end{align}

Moreover, every global minimizer satisfies
\begin{equation}\label{eq:optimizer_radius}
\fro{\Delta}^2\le(\eps^2+4)\fro{\A_{\Omc}\Mstar}^2\le D^2.
\end{equation}
\end{lemma}

\begin{lemma}[Monotonicity and lower bound for the bounded identifiability modulus]
\label{lem:mu_id_monotone_explicit}
The map $\rho\mapsto\mu_{\iid,D}(\rho)$ is nondecreasing on $(0,D]$. Assume there exist $\alpha_\star>0$ and $\rho_\star>0$ such that
\begin{equation}\label{eq:sharp_gap_assumption}
\fro{\A_{\Om}(M-\Mstar)}^2
\ge\alpha_\star\min\{\fro{M-\Mstar}^2,\rho_\star^2\}
\end{equation}
for every $M\succeq0$ with $\rank(M)\le r$ and $\fro{M-\Mstar}\le D$. Then, for every $0<\rho\le D$,
\begin{equation}\label{eq:mu_global_explicit}
\mu_{\iid,D}(\rho)\ge
\alpha_\star\min\!\left\{1,\frac{\rho_\star^2}{D^2}\right\}.
\end{equation}
\end{lemma}

\begin{remark}[Interpretation of the stronger assumption]
Assumption~\eqref{eq:sharp_gap_assumption} strengthens uniqueness on the bounded feasible class by requiring a quantitative observation gap: quadratic growth near $\Mstar$, and a positive separation beyond $\rho_\star$ within the $D$-ball. This assumption is used only for \eqref{eq:mu_global_explicit}; the recovery theorems use the bounded-modulus gap directly.
\end{remark}

\begin{lemma}[Concentration of missing energy]\label{lem:missing_energy}
Assume $\Om\sim \mathrm{Ber}(p)$ entrywise. For every $\chi>0$,
\begin{equation}\label{eq:miss_tail}
\Prob\!\left(
\left|\fro{\A_{\Omc}\Mstar}^2-(1-p)\fro{\Mstar}^2\right|> \chi\fro{\Mstar}^2
\right)
\le
2\exp\!\left[-c\min\!\left\{\frac{\chi^2 s_4(\Mstar)}{p(1-p)},\ \chi s_{\infty}(\Mstar)\right\}\right].
\end{equation}
In particular, with the same probability,
\begin{equation}\label{eq:miss_upper}
\fro{\A_{\Omc}\Mstar}^2\le (1-p+\chi)\fro{\Mstar}^2.
\end{equation}
\end{lemma}

For the shell $\cH(\rho)$, define the regularity parameters
\begin{equation}\label{eq:shell_regs}
s_{4,\rho}:=\inf_{H\in\cH(\rho)\setminus\{0\}}\frac{\fro{H}^4}{\four{H}^4},
\qquad
s_{\infty,\rho}:=\inf_{H\in\cH(\rho)\setminus\{0\}}\frac{\fro{H}^2}{\maxn{H}^2}.
\end{equation}

\begin{theorem}[Uniform Bernstein lower bound on the fixed shell]\label{thm:shell_bernstein}
For every $\chi\in(0,p)$,
\begin{equation}\label{eq:shell_bernstein_event}
\Prob\!\left(
\fro{\A_{\Om}H}^2\ge (p-\chi)\fro{H}^2
\quad\text{for all }H\in\cH(\rho)
\right)
\ge 1-\delta_{\mathrm{shell}}(\chi,\rho),
\end{equation}
where
\begin{equation}\label{eq:shell_bernstein_prob}
\delta_{\mathrm{shell}}(\chi,\rho)
:=
2\exp\!\left[(4n+2)r\log\!\left(\frac{36}{\chi}\right)
-c\min\!\left\{\frac{\chi^2 s_{4,\rho}}{p(1-p)},\ \chi s_{\infty,\rho}\right\}\right].
\end{equation}
\end{theorem}

\begin{proposition}[GNMR radius forces shell regularity]\label{prop:gnmr_shell_radius}
Assume $M^\star$ is rank $r$ and $\mu_0$-incoherent. Let
\[
\rho_{\mathrm{reg}}
:=
\min\!\left\{
\frac{\sigma_r(M^\star)}{2},\,
\frac{c_{\mu}\sqrt{\mu_0 r/n}}{1+c_{\mu}\sqrt{\mu_0 r/n}}\,\sigma_r(M^\star)
\right\},
\qquad
c_{\mu}:=\frac{\sqrt{3/2}-1}{\sqrt2}.
\]
If $\rho\le \rho_{\mathrm{reg}}$, then every nonzero $H=M-M^\star\in\cH(\rho)$ satisfies
\[
s_{\infty}(H)\ge \frac{n}{16\mu_0 r},
\qquad
s_4(H)\ge \frac{n}{16\mu_0 r}.
\]

\end{proposition}

\begin{corollary}[Explicit shell Bernstein bound under $\rho\le \rho_{\mathrm{GNMR}}$]\label{cor:shell_bernstein_gnmr}
Assume the hypotheses of Proposition~\ref{prop:gnmr_shell_radius}. Then, for every $\chi\in(0,p)$,
\begin{equation}\label{eq:shell_bernstein_gnmr}
\Prob\!\left(
\fro{\A_{\Om}H}^2\ge (p-\chi)\fro{H}^2
\quad\text{for all }H\in\cH(\rho)
\right)
\ge 1-\delta_{\mathrm{shell}}^{\mathrm{GNMR}}(\chi),
\end{equation}
where
\begin{equation}\label{eq:shell_bernstein_gnmr_prob}
\delta_{\mathrm{shell}}^{\mathrm{GNMR}}(\chi)
:=
2\exp\!\left[(4n+2)r\log\!\left(\frac{36}{\chi}\right)
-c\min\!\left\{\frac{\chi^2 n}{16p(1-p)\mu_0 r},\ \frac{\chi n}{16\mu_0 r}\right\}\right].
\end{equation}
\end{corollary}

\begin{lemma}[GNMR local RIP on the shell]\label{lem:gnmr_shell_rip}
Assume $\Mstar\in\Mset(n,n,r,\mu_0,\kappa)$ and the sample complexity condition \eqref{eq:gnmr_sample}. If $\rho\le \rho_{\mathrm{GNMR}}$, then with probability at least $1-3/n^3$,
\begin{equation}\label{eq:gnmr_shell_rip}
\fro{\A_{\Om}H}^2\ge p(1-\eta)^2\fro{H}^2
\qquad\text{for all }H\in\cH(\rho).
\end{equation}
\end{lemma}

\subsection{Missing Proofs}

\begin{proof}[Proof of Lemma~\ref{lem:epsilonMC-rip-general}]
By construction,
\[
\|\A_{\Om,\eps}(M)\|_2^2
=\eps^2\fro{M}^2+(1-\eps^2)\fro{\A_{\Om}M}^2.
\]
Since $0\le\fro{\A_{\Om}M}^2\le\fro{M}^2$, the deterministic bounds follow. Under Bernoulli sampling, $\E\fro{\A_{\Om}M}^2=p\fro{M}^2$, proving the expectation identity.
For fixed $M\ne0$, let $Z=\sum_{ij}(\delta_{ij}-p)M_{ij}^2$, where the $\delta_{ij}$ are independent Bernoulli variables. Its summands are bounded in magnitude by $B=\maxn{M}^2$, and its variance is at most $p(1-p)B\fro{M}^2$. One-sided Bernstein therefore gives, with probability at least $1-\delta$,
\[
Z\ge-\sqrt{2p(1-p)B\fro{M}^2u}-\frac{2Bu}{3}.
\]
Dividing by $\fro{M}^2$, combining with $\fro{\A_{\Om}M}^2\ge0$, and substituting into the first identity proves \eqref{eq:prob-rsc}.
\end{proof}

\begin{proof}[Proof of Lemma~\ref{lem:uniq_ident_equiv}]
For fixed $0<\rho\le D$, the annulus in \eqref{eq:rho_ident} is compact and the denominator is at least $\rho^2>0$. Uniqueness within the $D$-ball makes its continuous numerator strictly positive, so the ratio attains a positive minimum. Conversely, any distinct feasible completion at distance $d\in(0,D]$ has zero numerator and gives $\mu_{\iid,D}(d)=0$. Global unique completability implies uniqueness within this ball.
\end{proof}

\begin{proof}[Proof of Lemma~\ref{lem:basicineq}]
Since $\Mstar$ is feasible for \eqref{eq:eps_mc}, global optimality of $\Mdag$ yields
\[
\frac12\fro{\A_{\Om}(\Mdag-\Mstar)}^2+\frac{\eps^2}{2}\fro{\A_{\Omc}\Mdag}^2
\le
\frac12\fro{\A_{\Om}(\Mstar-\Mstar)}^2+\frac{\eps^2}{2}\fro{\A_{\Omc}\Mstar}^2.
\]
Since $\Mdag=\Delta+\Mstar$, this is exactly \eqref{eq:basicineq}. Dropping the second term proves \eqref{eq:obs_bound}; dropping the first proves \eqref{eq:unobs_Mdag}.

Finally, mask orthogonality gives $\fro{\Delta}^2=\fro{\A_{\Om}\Delta}^2+\fro{\A_{\Omc}\Delta}^2$. By the triangle inequality and \eqref{eq:unobs_Mdag}, $\fro{\A_{\Omc}\Delta}\le2\fro{\A_{\Omc}\Mstar}$, proving \eqref{eq:optimizer_radius}.
\end{proof}

\begin{proof}[Proof of Lemma~\ref{lem:mu_id_monotone_explicit}]
For $0<\rho_1\le\rho_2\le D$, the annulus with lower radius $\rho_2$ is a subset of that with lower radius $\rho_1$. Taking the infimum of the same ratio over the smaller set gives $\mu_{\iid,D}(\rho_1)\le\mu_{\iid,D}(\rho_2)$.
For any competitor in the annulus, write $d=\fro{M-\Mstar}$. Assumption~\eqref{eq:sharp_gap_assumption} gives a ratio at least $\alpha_\star$ when $d\le\rho_\star$, and at least $\alpha_\star\rho_\star^2/d^2\ge\alpha_\star\rho_\star^2/D^2$ otherwise. Taking the infimum proves \eqref{eq:mu_global_explicit}, including the case $\rho_\star\ge D$.
\end{proof}

\begin{proof}[Proof of Lemma~\ref{lem:missing_energy}]
Write
\[
\fro{\A_{\Omc}\Mstar}^2=
\sum_{i,j}(1-\delta_{ij}){\Mstar}_{ij}^2,
\qquad
\delta_{ij}\sim \mathrm{Ber}(p)\ \text{independently}.
\]
Define
\[
X_{ij}:=\big((1-\delta_{ij})-(1-p)\big){\Mstar}_{ij}^2.
\]
Then $\E[X_{ij}]=0$ and
\[
\sum_{i,j}X_{ij}=\fro{\A_{\Omc}\Mstar}^2-(1-p)\fro{\Mstar}^2.
\]
Moreover,
\[
|X_{ij}|\le \maxn{\Mstar}^2,
\qquad
\sum_{i,j}\mathrm{Var}(X_{ij})=p(1-p)\sum_{i,j}{\Mstar}_{ij}^4=p(1-p)\four{\Mstar}^4.
\]
Bernstein's inequality therefore gives
\[
\Prob\!\left(\left|\sum_{i,j}X_{ij}\right|>u\right)
\le
2\exp\!\left[-\frac{u^2/2}{p(1-p)\four{\Mstar}^4+\maxn{\Mstar}^2u/3}\right].
\]
Set $u=\chi\fro{\Mstar}^2$ and use
\[
\frac{\fro{\Mstar}^4}{\four{\Mstar}^4}=s_4(\Mstar),
\qquad
\frac{\fro{\Mstar}^2}{\maxn{\Mstar}^2}=s_{\infty}(\Mstar),
\]
to obtain \eqref{eq:miss_tail}. The upper bound \eqref{eq:miss_upper} follows immediately.
\end{proof}

\begin{proof}[Proof of Lemma~\ref{lem:transfer_shell}]
Suppose for contradiction that $\fro{\Delta}\ge \rho$. By \eqref{eq:optimizer_radius}, $\fro{\Delta}\le D$, so $\Mdag$ lies in the defining annulus. By definition of $\mu_{\iid,D}(\rho)$,
\[
\fro{\A_{\Om}\Delta}^2\ge \mu_{\iid,D}(\rho)\fro{\Delta}^2\ge \mu_{\iid,D}(\rho)\rho^2.
\]
On the other hand, \eqref{eq:obs_bound} from Lemma~\ref{lem:basicineq} gives
\[
\fro{\A_{\Om}\Delta}^2\le \eps^2\fro{\A_{\Omc}\Mstar}^2,
\]
contradicting \eqref{eq:transfer_cond}. Hence $\fro{\Delta}<\rho$.
\end{proof}

\begin{proof}[Proof of Theorem~\ref{thm:shell_bernstein}]
Let
\[
\cU(\rho):=\left\{\frac{H}{\fro{H}}:\ H\in\cH(\rho)\setminus\{0\}\right\}.
\]
Since every $H\in\cH(\rho)$ is a difference of two rank-$\le r$ matrices, one has $\rank(H)\le 2r$. Therefore $\cU(\rho)$ is contained in the unit Frobenius sphere of rank-$\le 2r$ matrices. A standard covering argument~\citep{candes2010tight} yields a $(\chi/4)$-net $\mathcal N\subseteq \cU(\rho)$ with cardinality
\[
|\mathcal N|\le \left(\frac{36}{\chi}\right)^{(4n+2)r}.
\]
Fix $G\in\mathcal N$. Since $\Om\sim\mathrm{Ber}(p)$ entrywise and $G$ is deterministic, the same scalar Bernstein argument as in Lemma~\ref{lem:missing_energy} gives
\[
\Prob\!\left(\fro{\A_{\Om}G}^2<(p-\chi/2)\fro{G}^2\right)
\le
2\exp\!\left[-c\min\!\left\{\frac{\chi^2 s_4(G)}{p(1-p)},\ \chi s_{\infty}(G)\right\}\right].
\]
Since $G\in\cU(\rho)$, we have $s_4(G)\ge s_{4,\rho}$ and $s_{\infty}(G)\ge s_{\infty,\rho}$. Union bound over $\mathcal N$ yields
\[
\Prob\!\left(\fro{\A_{\Om}G}^2\ge (p-\chi/2)\fro{G}^2\ \text{for all }G\in\mathcal N\right)
\ge 1-\delta_{\mathrm{shell}}(\chi,\rho).
\]
Now let $H\in\cU(\rho)$ and choose $G\in\mathcal N$ with $\fro{H-G}\le \chi/4$. Then, on the event that
\[
\fro{\A_{\Om}G}^2\ge (p-\chi/2)\fro{G}^2
\qquad\text{for all }G\in\mathcal N,
\]
we compare $H$ to its nearby net point $G$. Since $H,G\in\cU(\rho)$, we have
\[
\fro{H}=\fro{G}=1.
\]
Also, for every matrix $Z$,
\[
\fro{\A_{\Om}Z}\le \fro{Z},
\]
because $\A_{\Om}$ only keeps a subset of the entries of $Z$. Using
\[
\bigl|\|X\|_F^2-\|Y\|_F^2\bigr|
=
\bigl|\|X\|_F-\|Y\|_F\bigr|\bigl(\|X\|_F+\|Y\|_F\bigr)
\le
\|X-Y\|_F\bigl(\|X\|_F+\|Y\|_F\bigr),
\]
with $X=\A_{\Om}H$ and $Y=\A_{\Om}G$, we obtain
\begin{align*}
\big|\fro{\A_{\Om}H}^2-\fro{\A_{\Om}G}^2\big|
&\le
\fro{\A_{\Om}(H-G)}
\bigl(\fro{\A_{\Om}H}+\fro{\A_{\Om}G}\bigr) \\
&\le
\fro{H-G}\bigl(\fro{H}+\fro{G}\bigr) \\
&\le
\frac{\chi}{4}\cdot 2
=
\frac{\chi}{2}.
\end{align*}
Hence, still on the net event,
\[
\fro{\A_{\Om}H}^2
\ge
\fro{\A_{\Om}G}^2-\frac{\chi}{2}
\ge
\left(p-\frac{\chi}{2}\right)\fro{G}^2-\frac{\chi}{2}
=
p-\chi.
\]
Since $H\in\cU(\rho)$ was arbitrary, this shows that
\[
\fro{\A_{\Om}H}^2\ge p-\chi
\qquad\text{for all }H\in\cU(\rho)
\]
with probability at least $1-\delta_{\mathrm{shell}}(\chi,\rho)$. Scaling back from $\cU(\rho)$ to $\cH(\rho)$ proves \eqref{eq:shell_bernstein_event}.
\end{proof}

\begin{proof}[Proof of Proposition~\ref{prop:gnmr_shell_radius}]
Let $0\ne H=M-\Mstar\in\cH(\rho)$, put $h=\fro{H}$, $a=\mu_0r/n$, and $\sigma=\sigma_r(\Mstar)$. We show $\maxn{H}\le4\sqrt a\,h$. If $a\ge1/2$, this follows directly from $\maxn{H}\le h$. Otherwise $\mu_0\ge1$ implies $n>2r$. Weyl's inequality and $h\le\rho\le\sigma/2$ imply $\rank(M)=r$.
Write $\Mstar=U_\star\Lambda_\star U_\star^T$ and $M=U\Lambda U^T$. The identity
\[
(I-U_\star U_\star^T)U\Lambda=(I-U_\star U_\star^T)HU
\]
gives
\[
\|\sin\Theta(U,U_\star)\|_2\le\frac{\rho}{\sigma-\rho}\le c_\mu\sqrt a.
\]
Choose principal-vector bases, so that $U=U_\star C+U_\perp S$, where $C,S$ are the diagonal cosine and sine matrices and $[U_\star\ U_\perp]$ has orthonormal columns. Such a completion is possible since $n>2r$. These rotations preserve row norms and conjugate $\Lambda_\star,\Lambda$. Then
\[
H=[U_\star\ U_\perp]
\begin{bmatrix}A&B\\B^T&D\end{bmatrix}
[U_\star\ U_\perp]^T,
\quad A=C\Lambda C-\Lambda_\star,\quad B=C\Lambda S,\quad D=S\Lambda S.
\]
Orthogonality gives $h^2=\fro A^2+2\fro B^2+\fro D^2$, so $\op A,\op B\le h$. Since $\op S\le c_\mu\sqrt a<1/2$, $C$ is invertible and $\op{C^{-1}}\le2$. The exact identity $D=SC^{-1}B$ implies $\op D\le2c_\mu\sqrt a\,h$.
Using $\|e_i^TU_\star\|_2\le\sqrt a$ and $\|e_i^TU_\perp\|_2\le1$ separately on each block gives
\[
|H_{ij}|\le ah+2\sqrt a\,h+2c_\mu\sqrt a\,h
\le(3+2c_\mu)\sqrt a\,h\le4\sqrt a\,h.
\]
Thus $s_\infty(H)\ge1/(16a)=n/(16\mu_0r)$. Finally, $\four H^4\le\maxn H^2\fro H^2$ gives $s_4(H)\ge s_\infty(H)$.
\end{proof}

\begin{proof}[Proof to Lemma~\ref{lem:gnmr_shell_rip}]
Recall the GNMR local-RIP neighborhood defined in \cite{zilber2022gnmr}:
\[
\mathcal N_{\mathrm{GNMR}}(\eta)
:=
\Big\{
M\in\R^{n\times n}:
\rank(M)=r,\ 
M \text{ is } \tfrac{3}{2}\mu_0\text{-incoherent},\
\fro{M-\Mstar}\le \tfrac{\eta\sigma_r^\star}{c_e}
\Big\},
\]
where $\sigma_r^\star:=\sigma_r(\Mstar)$ and $c_e>0$ is the absolute constant from the GNMR theorem.  
By the GNMR local RIP result, under the sample complexity condition \eqref{eq:gnmr_sample}, there exists an event $\mathcal E_{\mathrm{GNMR}}$ with
\[
\Prob(\mathcal E_{\mathrm{GNMR}})\ge 1-\frac{3}{n^3}
\]
such that on $\mathcal E_{\mathrm{GNMR}}$,
\begin{equation}\label{eq:gnmr_local_rip_used}
\fro{\A_{\Om}(M-\Mstar)}^2
\ge
p(1-\eta)^2\fro{M-\Mstar}^2
\qquad
\text{for all }M\in \mathcal N_{\mathrm{GNMR}}(\eta).
\end{equation}

We will show that if $\rho\le \rho_{\mathrm{GNMR}}$, then every shell element $H\in \cH(\rho)$ satisfies $M:=\Mstar+H\in \mathcal N_{\mathrm{GNMR}}(\eta)$, and then \eqref{eq:gnmr_local_rip_used} will imply the claim.

Let $H\in \cH(\rho)$ and define
\[
M:=\Mstar+H.
\]
By definition of the shell,
\[
\fro{M-\Mstar}=\fro{H}\le \rho.
\]
Since $\rho\le \rho_{\mathrm{GNMR}}\le \eta\sigma_r^\star/c_e$, it follows that
\begin{equation}\label{eq:radius_check}
\fro{M-\Mstar}\le \frac{\eta\sigma_r^\star}{c_e}.
\end{equation}

Because $H\in \cH(\rho)$, the associated matrix $M=\Mstar+H$ belongs to the admissible model class, so in particular
\[
\rank(M)\le r.
\]
Also, since $\Mstar$ has rank $r$, Weyl's inequality yields
\[
\sigma_r(M)
\ge
\sigma_r(\Mstar)-\op{M-\Mstar}
\ge
\sigma_r^\star-\fro{M-\Mstar}
\ge
\sigma_r^\star-\rho.
\]
Using $\rho\le \rho_{\mathrm{GNMR}}\le \sigma_r^\star/2$, we obtain
\[
\sigma_r(M)\ge \frac{\sigma_r^\star}{2}>0.
\]
Hence $\rank(M)\ge r$. Combined with $\rank(M)\le r$, this gives
\begin{equation}\label{eq:rank_check}
\rank(M)=r.
\end{equation}

Write the eigendecompositions
\[
\Mstar = U_\star\Lambda_\star U_\star^\top,
\qquad
M = U\Lambda U^\top,
\]
where $U_\star,U\in\R^{n\times r}$ have orthonormal columns. Since $\Mstar\in \Mset(n,n,r,\mu_0,\kappa)$, we know that
\[
\max_{1\le i\le n}\|e_i^\top U_\star\|_2^2 \le \frac{\mu_0 r}{n}.
\]

By Davis--Kahan,
\[
\|\sin\Theta(U,U_\star)\|_{\mathrm{op}}
\le
\frac{\op{M-\Mstar}}{\sigma_r^\star-\rho}
\le
\frac{\fro{M-\Mstar}}{\sigma_r^\star-\rho}
\le
\frac{\rho}{\sigma_r^\star-\rho}.
\]
Therefore there exists an orthogonal matrix $Q\in\R^{r\times r}$ such that
\[
\op{U-U_\star Q}
\le
\sqrt{2}\,\|\sin\Theta(U,U_\star)\|_{\mathrm{op}}
\le
\sqrt{2}\,\frac{\rho}{\sigma_r^\star-\rho}.
\]
Fix any row index $i$. Then
\[
\|e_i^\top U\|_2
\le
\|e_i^\top U_\star Q\|_2 + \|e_i^\top(U-U_\star Q)\|_2
\le
\sqrt{\frac{\mu_0 r}{n}} + \sqrt{2}\,\frac{\rho}{\sigma_r^\star-\rho}.
\]
Now choose the constant hidden in $\rho_{\mathrm{GNMR}}$ so that
\[
\rho_{\mathrm{GNMR}}
\le
\frac{c_\mu\sqrt{\mu_0 r/n}}{1+c_\mu\sqrt{\mu_0 r/n}}\sigma_r^\star,
\qquad
c_\mu \le \frac{\sqrt{3/2}-1}{\sqrt{2}}.
\]
Then, since $\rho\le \rho_{\mathrm{GNMR}}$,
\[
\sqrt{2}\,\frac{\rho}{\sigma_r^\star-\rho}
\le
\Big(\sqrt{\frac32}-1\Big)\sqrt{\frac{\mu_0 r}{n}}.
\]
Substituting into the previous row bound gives
\[
\|e_i^\top U\|_2
\le
\sqrt{\frac{\mu_0 r}{n}}
+
\Big(\sqrt{\frac32}-1\Big)\sqrt{\frac{\mu_0 r}{n}}
=
\sqrt{\frac{3\mu_0 r}{2n}}.
\]
Since this holds for every $i$,
\begin{equation}\label{eq:incoherence_check}
\max_{1\le i\le n}\|e_i^\top U\|_2^2 \le \frac{3\mu_0 r}{2n},
\end{equation}
so $M$ is $(3/2)\mu_0$-incoherent. From \eqref{eq:radius_check}, \eqref{eq:rank_check}, and \eqref{eq:incoherence_check}, we conclude that
\[
M=\Mstar+H\in \mathcal N_{\mathrm{GNMR}}(\eta)
\qquad\text{for every }H\in \cH(\rho).
\]
Therefore, on the event $\mathcal E_{\mathrm{GNMR}}$, applying \eqref{eq:gnmr_local_rip_used} to $M=\Mstar+H$ yields
\[
\fro{\A_{\Om}H}^2
=
\fro{\A_{\Om}(M-\Mstar)}^2
\ge
p(1-\eta)^2\fro{M-\Mstar}^2
=
p(1-\eta)^2\fro{H}^2
\]
for all $H\in \cH(\rho)$. This is exactly \eqref{eq:gnmr_shell_rip}.
\end{proof}

\begin{proof}[Proof of Theorem~\ref{thm:main_bernoulli}]
Let
\[
\cE_{\mathrm{miss}}(\chi):=\{\fro{\A_{\Omc}\Mstar}^2\le (1-p+\chi)\fro{\Mstar}^2\}
\]
and
\[
\cE_{\mathrm{shell}}(\chi):=\left\{\fro{\A_{\Om}H}^2\ge (p-\chi)\fro{H}^2\ \text{for all }H\in\cH(\rho)\right\}.
\]
By Lemma~\ref{lem:missing_energy},
\[
\Prob(\cE_{\mathrm{miss}}(\chi))
\ge
1-2\exp\!\left[-c\min\!\left\{\frac{\chi^2 s_4(\Mstar)}{p(1-p)},\ \chi s_{\infty}(\Mstar)\right\}\right].
\]
By Theorem~\ref{thm:shell_bernstein},
\[
\Prob(\cE_{\mathrm{shell}}(\chi))\ge 1-\delta_{\mathrm{shell}}(\chi,\rho).
\]
On $\cE_{\mathrm{miss}}(\chi)$, assumption \eqref{eq:main_mu_gap} implies
\[
\mu_{\iid,D}(\rho)\rho^2 > \eps^2\fro{\A_{\Omc}\Mstar}^2.
\]
Hence Lemma~\ref{lem:transfer_shell} gives $\Delta\in\cH(\rho)$. On the event $\cE_{\mathrm{shell}}(\chi)$,
\[
\fro{\A_{\Om}\Delta}^2\ge (p-\chi)\fro{\Delta}^2.
\]
Combining this with \eqref{eq:obs_bound} from Lemma~\ref{lem:basicineq} and the bound from $\cE_{\mathrm{miss}}(\chi)$ gives
\[
(p-\chi)\fro{\Delta}^2
\le \eps^2\fro{\A_{\Omc}\Mstar}^2
\le \eps^2(1-p+\chi)\fro{\Mstar}^2,
\]
which proves \eqref{eq:main_bernoulli_bound}. The probability bound follows from a union bound.
\end{proof}

\begin{proof}[Proof of Theorem~\ref{thm:main_gnmr}]
Let $\cE_{\mathrm{miss}}(\chi)$ be as in the previous proof. By Lemma~\ref{lem:missing_energy},
\[
\Prob(\cE_{\mathrm{miss}}(\chi))
\ge
1-2\exp\!\left[-c\min\!\left\{\frac{\chi^2 s_4(\Mstar)}{p(1-p)},\ \chi s_{\infty}(\Mstar)\right\}\right].
\]
On this event, assumption (iv) implies
\[
\mu_{\iid,D}(\rho)\rho^2 > \eps^2\fro{\A_{\Omc}\Mstar}^2,
\]
so Lemma~\ref{lem:transfer_shell} gives $\Delta\in\cH(\rho)$. Since $\rho\le \rho_{\mathrm{GNMR}}$, Corollary~\ref{lem:gnmr_shell_rip} yields, with probability at least $1-3/n^3$,
\[
\fro{\A_{\Om}\Delta}^2\ge p(1-\eta)^2\fro{\Delta}^2.
\]
Combining this with \eqref{eq:obs_bound} from Lemma~\ref{lem:basicineq} and the event $\cE_{\mathrm{miss}}(\chi)$ gives
\[
p(1-\eta)^2\fro{\Delta}^2
\le \eps^2\fro{\A_{\Omc}\Mstar}^2
\le \eps^2(1-p+\chi)\fro{\Mstar}^2,
\]
which is \eqref{eq:main_gnmr_bound}. The probability bound follows by intersecting the GNMR event with $\cE_{\mathrm{miss}}(\chi)$.
\end{proof}

\begin{proof}[Proof of Theorem~\ref{thm:socp_lifted_noisy}]
First of all, we hope to decompose the Hessian of \eqref{eq:main_ms} at a second order point $\hat X \in \RR^{n \times r}$. Classic matrix sensing literatures like \cite{ha2020equivalence,zhang2021general,li2019non} give that the second-order critical condition of \eqref{eq:main_ms} are given as 
\begin{gather}\label{eq:focp_unlifted}
		\nabla f(\hat X \hat X^\top)\hat X = 0, \\
		\label{eq:socp_unlifted}
		2 \langle \nabla f(\hat X \hat X^\top), U U^\top \rangle + [\nabla^2 f(\hat X \hat X^\top)](\hat X U^\top + U \hat X^\top,\hat X U^\top + U \hat X^\top) \geq 0 \quad \forall U \in \RR^{n \times r}
\end{gather}
	with \eqref{eq:focp_unlifted} being the first order critical condition. Moreover, since the sensing matrices $\{A_i\}_{i \in [m]}$ can be assumed be to symmetric without loss of generality \citep{zhang2021general}, we have that 
	\[
		[\nabla^2 f(\hat X \hat X^\top)](\hat X U^\top + U \hat X^\top,\hat X U^\top + U \hat X^\top) = 4 [\nabla^2 f(\hat X \hat X^\top)](\hat X U^\top,\hat X U^\top).
	\]
	We then could decompose LHS of \eqref{eq:socp_unlifted} as $2C_1+4C_2$ where
	\[
		C_1 \coloneqq \langle \nabla f(\hat X \hat X^\top), U U^\top \rangle, \quad C_2 \coloneqq [\nabla^2 f(\hat X \hat X^\top)](\hat X U^\top,\hat X U^\top)
	\]
Given the assumption that \eqref{eq:main_ms} obeys some RSS condition, it is possible to upper-bound $C_2$ by observing
\[
	[\nabla^2 f(\hat X \hat X^\top)](\hat X U^\top + U \hat X^\top,\hat X U^\top + U \hat X^\top) \leq L_s \|\hat X U^\top + U \hat X^\top\|^2_F
\]
Therefore, if want to somehow create an negative escape direction for $\hat X$, it is important that we find a $U$ such that $C_1$ is negative and large in magnitude, and then amplify this term via tensor parametrization. To do so, we first do a more in-depth analysis of $\nabla f(\hat X \hat X^\top)$. As mentioned above, since $\nabla f(\cdot)$ can be assumed to be symmetric, one can select $u \in \RR^n$ such that $u^\top \nabla f(\hat x \hat x^\top)u = \lambda_{\text{min}}(\nabla f(\hat x \hat x^\top))$. Then via the definition of RSC we have
	\begin{equation}\label{eq:taylor}
		f(M^*) \geq f(\hat X \hat X^\top) + \langle \nabla f(\hat X \hat X^\top ), M^* - \hat X \hat X^\top \rangle + \frac{\alpha_s}{2}  \|\hat X \hat X^\top -M^*\|^2_F.
	\end{equation}
With $\hat X$ being a first-order point, according to \eqref{eq:focp_unlifted}
	\[
		\nabla f(\hat X \hat X^\top)\hat X = 0 \implies \langle \nabla f(\hat X \hat X^\top ), \hat X \hat X^\top \rangle = 0
	\]
Therefore, if in \eqref{eq:main_ms} our $b$ is corrupted as $\mathcal{A}(M^*)+\tilde w$, then plugging it back into \eqref{eq:taylor} gives 
\begin{equation}
	\begin{aligned}
		\langle \nabla f(\hat X \hat X^\top ), M^* \rangle &\leq -\frac{\alpha_s}{2} \|\hat x \hat x^\top -M^*\|^2_F + f(M^*) - f(XX^\top) \\
		&\leq -\frac{\alpha_s}{2} \|\hat x \hat x^\top -M^*\|^2_F + f(M^*) \\
		&=  -\frac{\alpha_s}{2} \|\hat x \hat x^\top -M^*\|^2_F + \frac{ \|\tilde w\|^2_2}{2}
	\end{aligned}
\end{equation}
where the second inequality follows from the fact that $f(\cdot) \geq 0$ in its entire domain and the last inequality follows from $f(M^*) = 1/2 \|\mathcal{A}(M^*) - \mathcal{A}(M^*) - \tilde w\|^2_2 = \|\tilde w\|^2_2/2$. Furthermore, since both $\nabla f(\hat X \hat X^\top )$ and $M^*$ are assumed to be positive semidefinite, \[
		\langle \nabla f(\hat X \hat X^\top ), M^* \rangle \geq \lambda_{\min} (\nabla f(\hat X \hat X^\top )) \tr(M^*)
	\]
	which implies that 
	\begin{equation}
		\lambda_{\min} (\nabla f(\hat X \hat X^\top )) \leq  \frac{ -\alpha_s\|\hat X \hat X^\top -M^*\|^2_F + \|\tilde w\|^2_2}{2\tr(M^*)}
	\end{equation}
	Furthermore, \eqref{eq:thm_socp_distance_lower_highr} gives us
	\[
		\|\hat X \hat X^\top -M^*\|^2_F \geq \|\tilde w\|^2_2/\alpha_s
	\]
	since $\frac{L_s}{\alpha_s} \lambda_r(\hat X \hat X^\top) \tr(M^*) \geq 0$ by definition. This means that
	\begin{equation}\label{eq:G_bound}
		\lambda_{\min} (\nabla f(\hat X \hat X^\top )) \leq  \frac{ -\alpha_s\|\hat X \hat X^\top -M^*\|^2_F + \|\tilde w\|^2_2}{2\tr(M^*)} \leq 0
	\end{equation} 
Thus, with this result equipped, we can further find a $U$ that makes $C_1$ small. In the most convenient manner, we first consider the eigenvector $u \in \RR^n$ of $\nabla f(\hat X \hat X^\top )$ associated with $\lambda_{\min} (\nabla f(\hat X \hat X^\top ))$. Additionally we consider $q \in \RR^{r}$ to be the $r$-th singular value of $\hat X$, with
	\[
		\|\hat X q\|_2 = \sigma_r(\hat X), \qquad \|q\|_2=1
	\]
	Then choosing $U \in \RR^{n \times r} = uq^\top$ leads to
	\[
		C_1 = \langle \nabla f(\hat X \hat X^\top), U U^\top \rangle = \langle \nabla f(\hat X \hat X^\top), u u^\top \rangle = -G
	\]
	where $G \coloneqq - \lambda_{\text{min}}(\nabla f(\hat X \hat X^\top)) \geq 0$. By recalling $\hat X^\top u = 0$ according to the first-order condition \eqref{eq:focp_unlifted}, we can further bound $C_2$ with this chocie of $U$ as
	\begin{equation*}
		\begin{aligned}
			 L_s \|\hat X U^\top + U \hat X^\top\|^2_F  
			&= L_s \|u(\hat X q)^\top + (\hat X q)u^\top \|^2_F \\
			&= 2 L_s \|\hat X q\|^2_F + 2 L_s (q^\top (\hat X^\top u))^2 \\
			&= 2 L_s \lambda_r(\hat X \hat X^\top),
		\end{aligned}
	\end{equation*}
	leading to
	\[
		C_2 \leq \frac{1}{2} L_s \lambda_r(\hat X \hat X^\top)
	\]
Now, if we choose $\Delta = \vecc(U)^{\otimes l}$ for the aforementioned $U \in \RR^{n \times r}$, the LHS of \eqref{eq:socp_lifted_highr} can be expressed as:
	\begin{equation}\label{eq:thm_socp_highr_help1}
		\begin{aligned}
			&2 ( \langle \mb{A}, \hat X \hat X^\top \rangle_{2,3}^\top \langle \mb{A}, uu^\top \rangle_{2,3})^{l} - 2 \left( (\langle \mb{A}, M^* \rangle_{2,3} +\tilde w)^\top \langle \mb{A}, uu^\top \rangle_{2,3} \right)^{l} + 4 ( \|\langle \mb{A}, \hat X U^\top \rangle_{2,3}\|^2_2)^{l} \\
			\leq  &2 (\lambda_{\text{min}}(\nabla f(\hat X \hat X^\top)))^l + 4 C_2^l \\
			=  &2 C_1^l + 4 C_2^l
		\end{aligned}
	\end{equation}
	where the inequality follows from:
	\[
		a^n - b^n \leq (a-b)^n, \quad \forall b \geq a \geq 0
	\]
	Here, since $a-b = C_1 \leq 0$, the above inequality can be used. As a result,
	\[
		\text{LHS of \eqref{eq:socp_lifted_highr}} \leq \underbrace{-2G^l}_{\text{Part 1}} + \underbrace{\frac{2}{2^{l-1}} L^l_s \lambda_r(\hat X \hat X^\top)^l}_{\text{Part 2}}
	\]
	We know since $G \geq 0$, Part 1 is always negative assuming $l$ is odd, and Part 2 is always positive. Therefore, it suffices to find an order $l$ such that 
	\begin{equation}\label{eq:l_condition_highr}
		G^l > (1/2^{l-1}) L^l_s \lambda_r(\hat X \hat X^\top)^l
	\end{equation}
	Conveniently, \eqref{eq:G_bound} says that 
	\begin{equation}\label{eq:G_lower_highr}
		G \geq \frac{\alpha_s \|M^* - \hat X \hat X^\top \|^2_F - \|\tilde w\|^2_2}{2 \tr(M^*)},
	\end{equation}
	which can be used to derive sufficient condition for \eqref{eq:l_condition_highr}. Therefore, if
	\[
		\left( \frac{\alpha_s \|M^* - \hat X \hat X^\top \|^2_F - \|\tilde w\|^2_2}{2 \tr(M^*)} \right)^l > (1/2^{l-1}) L^l_s \lambda_r(\hat X \hat X^\top)^l,
	\]
	we can conclude that \eqref{eq:l_condition_highr} holds, which implies that the LHS of \eqref{eq:socp_lifted_highr} is negative, directly proving that $\hat X^{\otimes l}$ is not an SOP anymore. Elementary manipulations of the above equation give that a sufficient condition is 
	\begin{equation}\label{eq:l_condition2}
		\|M^* - \hat X \hat X^\top \|^2_F - \|\tilde w\|^2_2/\alpha_s > 2^{1/l} \frac{L_s}{\alpha_s} \lambda_r(\hat X \hat X^\top) \tr(M^*)
	\end{equation}
	We now consider \eqref{eq:thm_socp_distance_lower_highr}, which means that
	\begin{equation}\label{eq:xhat_upper_D_highr}
		\lambda_r(\hat X \hat X^\top)  \leq \frac{\alpha_s \|M^* - \hat X \hat X^\top \|^2_F -\|\tilde w\|^2_2}{L_s \tr(M^*)}
	\end{equation}
	Subsequently, define a constant $\gamma$ such that
	\[
		L_s \lambda_r (\hat X \hat X^\top) = \gamma \left[\frac{\alpha_s \|M^* - \hat X \hat X^\top \|^2_F -\|\tilde w\|^2_2}{2 \tr(M^*)} \right]
	\]
	Then, \eqref{eq:G_lower_highr} and \eqref{eq:xhat_upper_D_highr} together imply that $1 \leq \gamma <2$. Using this simplified notation, our sufficient condition \eqref{eq:l_condition2} becomes
	\begin{equation}\label{eq:l_condition3_highr}
		1 > \frac{\gamma}{2^{(l-1)/l}}
	\end{equation}
	Given $1 \leq \gamma < 2$, there always exists a large enough $l$ such that \eqref{eq:l_condition3_highr} holds, which proves that LHS of \eqref{eq:socp_lifted_highr} is negative, proving that $\vecc(\hat X)^{\otimes l}$ is a strict saddle, concluding the proof.
	
	To derive a sufficient $l$, we simply acknowledge
	\[
		\gamma = \frac{2 L_s \tr(M^*) \lambda_r(\hat X \hat X^\top)}{\alpha_s  \|M^* - \hat X \hat X^\top \|^2_F -\|\tilde w\|^2_2 } \coloneqq 2 \beta
	\]
	and that $\beta \leq 1$ due to assumption \eqref{eq:thm_socp_distance_lower_highr}. Therefore, for \eqref{eq:l_condition3_highr} to hold true, it is enough to have 
	\[
		2^{(l-1)/l} > 2 \beta \implies \frac{l-1}{l} > \log_2(2 \beta) \implies l > \frac{1}{1-\log_2(2 \beta)} 
	\]
\end{proof}

\begin{proof}[Proof of Theorem~\ref{thm:implicit_bias_gd}]
	First of all, we hope to decompose the GD trajectory of \eqref{eq:lifted_main_highr} $\{\mb{w}_t\}_{t =0}^T$ as follows:
	\begin{equation}\label{eq:wt_breakdown}
		\mb{w}_{t+1} = \langle \mb{Z}_t, \mb{w}_0 \rangle - \mb{E}_t \coloneqq \mb{\tilde w}_t - \mb{E}_t
	\end{equation}
	where 
	\begin{align*}
		\mb{Z}_t &\coloneqq  (\mathcal{I}+ \eta \langle \mb{A}^{\otimes l}_r ,\tilde b^{\otimes l} \rangle)^t, \quad \mb{A}_r = I_r \oslash_{2,3} A \\
		\mb{E}_t &\coloneqq \sum_{i=1}^t (\mathcal{I}+ \eta \langle \mb{A}^{\otimes l}_r ,\tilde b^{\otimes l} \rangle)^{t-i} \mb{\hat E}_i \\
		 \mb{\hat E}_i &\coloneqq \eta \langle \langle (\mb{A}^l_r)^* \mb{A}^l, \langle \mb{P}(\mb{w}_{i-1}), \mb{P}(\mb{w}_{i-1})\rangle_{2*[l]} \rangle, \mb{w}_{i-1} \rangle_{2*[l]} \\
		 (\mb{A}^l_r)^* \mb{A}^l &\coloneqq \langle (\mb{A}_r)^{\otimes l}, \mb{A}^{\otimes l} \rangle_{3,6,\dots,3l} \in \RR^{[nr \times nr \times n \times n] \circ l}
	\end{align*}
	This can be proved via induction where
	\begin{align*}
		\mb{w}_1 &= \left(\mathcal{I}+\eta 
	\langle \mb{A}^{\otimes l}_r, \tilde b^{\otimes l} -  (\mb{A}^{\otimes l})^*\langle \mb{P}(\mb{w}_0), \mb{P}(\mb{w}_0)\rangle \right) \mb{w}_0 \\
		&= (\mathcal{I}+\eta \langle \mb{A}^{\otimes l}_r ,\tilde b^{\otimes l} \rangle) \mb{w}_0 - \eta \langle (\mb{A}^l_r)^* \mb{A}^l, \langle \mb{P}(\mb{w}_0), \mb{P}(\mb{w}_0)\rangle \rangle \mb{w}_0 \\
		&= \langle \mb{Z}_1, \mb{w}_0 \rangle - \mb{E}_1
	\end{align*}
	This serves as our base case, and the induction step can be proven as 
	\begin{align*}
		\mb{w}_{t+1} &= \left(\mathcal{I}+\eta 
	\langle \mb{A}^{\otimes l}_r, \tilde b^{\otimes l} -  (\mb{A}^{\otimes l})^*\langle \mb{P}(\mb{w}_t),\mb{P}(\mb{w}_t)\rangle \right) \mb{w}_t \\
		&= (\mathcal{I}+\eta \langle \mb{A}^{\otimes l}_r ,\tilde b^{\otimes l} \rangle) \mb{w}_t - \eta \langle (\mb{A}^l_r)^* \mb{A}^l, \langle \mb{P}(\mb{w}_t), \mb{P}(\mb{w}_t)\rangle \rangle \mb{w}_t  \\
		& = (\mathcal{I}+\eta \langle \mb{A}^{\otimes l}_r ,\tilde b^{\otimes l} \rangle) \mb{w}_t - \mb{\hat E}_{t+1} \\
		&= (\mathcal{I}+\eta \langle \mb{A}^{\otimes l}_r ,\tilde b^{\otimes l} \rangle) \left( \mb{\tilde w}_t - \sum_{i=1}^t (\mathcal{I}+ \eta \langle \mb{A}^{\otimes l}_r ,\tilde b^{\otimes l} \rangle)^{t-i} \mb{\hat E}_i \right) - \mb{\hat E}_{t+1} \\
		&= \mb{\tilde w}_{t+1} - \sum_{i=1}^t (\mathcal{I}+ \eta \langle \mb{A}^{\otimes l}_r ,\tilde b^{\otimes l} \rangle)^{t+1-i} \mb{\hat E}_i- \mb{\hat E}_{t+1} \\
		&= \mb{\tilde w}_{t+1} - \sum_{i=1}^{t+1} (\mathcal{I}+ \eta \langle \mb{A}^{\otimes l}_r ,\tilde b^{\otimes l} \rangle)^{t+1-i} \mb{\hat E}_i \\
		&= \mb{\tilde w}_{t+1} - \mb{E}_t
	\end{align*}
	Therefore, we can then use a version of Lemma 13 and Lemma 2 in \cite{ma2024algorithmic} with $U = \langle  \mb{A}_r^* \mb{A}, M^* \rangle$ replaced with $U = \langle \mb{A}_r, \tilde b \rangle_{1}$ to prove this theorem, following the steps in proof to Theorem 1 in \cite{ma2024algorithmic}.
\end{proof}

\section{Spatial distribution of reconstruction error}
\label{app:error-spatial}
Global metrics such as the Frobenius error
$\|M^{\text{sol}} - M^\star\|_F$ summarize the quality of reconstruction in a single
number, but they do not reveal how errors are distributed across the matrix. In such cases, it
is natural to ask whether failure manifests as a few ``bad'' rows or columns,
or whether the algorithm fits the observed entries well but extrapolates
incorrectly on a subset of unobserved entries.

To investigate this question, we analyze the entrywise squared error
$
E_{ij}^2 = (M^{\text{sol}}_{ij} - M^\star_{ij})^2
$
and its spatial distribution over the matrix. We consider the structured instance of problem~\eqref{eq:structured_prob} and use our method described in Theorem~\ref{thm:main_result} to obtain reconstructed solutions $M^{\text{sol}}$. We performed numerical studies on $E$, and have made some interesting observations. A striking feature of all unsuccessful runs is that the error on observed entries is essentially zero. The optimization
procedure consistently recovers the observed data very accurately. The
non-negligible reconstruction error is confined to unobserved entries, indicating
that the algorithm converges to stationary points that are locally optimal with
respect to the observed data, rather than diverging or producing a globally
incoherent factorization.

Table~\ref{tab:mse-obs-unobs} showcases this phenomenon quantitatively.
Across 10 independent runs, we report the mean-squared error on observed and
unobserved entries, denoted $\text{MSE}_{\text{obs}}$ and
$\text{MSE}_{\text{unobs}}$, respectively. In 10 out of 10 runs, $\text{MSE}_{\text{obs}}$ is on the order of $10^{-8}$ or even lower, while $\text{MSE}_{\text{unobs}}$ is solely responsible in determining whether the reconstruction was successful or not. This shows a gap of roughly eight orders of magnitude between the observed and unobserved regions. To take a closer look at the results, we selected 4 runs out of the 10 and plotted their entrywise reconstruction errors using heat-maps in Figure~\ref{fig:structured-error-maps}. We overlayed the observation mask~$\Omega$ with white squares indicating the observed entries, while the background color encodes the squared reconstruction error.

From Figure~\ref{fig:structured-error-maps} we observe that the error patterns on the unobserved entries differ substantially
across the three unsuccessful runs. In each case, different subsets of the
unobserved region carry the dominant error, and there is no single block or
band that is systematically problematic across runs. This behavior is consistent
with the presence of many spurious local minima in
problem~\eqref{eq:structured_prob}: different random initializations lead
gradient descent to distinct low-rank completions that interpolate the observed
entries but disagree on how to fill in the missing part. In contrast, the
successful run exhibits uniformly small squared error on both observed and
unobserved entries, corresponding to convergence to the global minimum. Taken
together, these experiments suggest that unsuccessful recoveries are primarily due to convergence to benign local minima that fit the data but differ in their extrapolation to unobserved entries, rather than to systematic failure in any specific region of the matrix.

\begin{table}[H]
    \centering
    \setlength{\tabcolsep}{2pt}
    \begin{tabular}{l*{10}{c}}
        \toprule
        & s0 & s1 & s2 & s3 & s4 & s5 & s6 & s7 & s8 & s9 \\
        \midrule
        $\mathrm{MSE}_{\text{obs}}$   &
        3.26e-8 & 4.89e-8 & 9.46e-9 & 9.79e-9 & 1.35e-8 &
        1.67e-8 & 1.07e-8 & 9.40e-9 & 1.41e-8 & 1.50e-8 \\
        $\mathrm{MSE}_{\text{unobs}}$ &
        2.67e0 & 4.95e-8 & 2.67e0 & 2.67e0 & 8.23e-8 &
        8.24e-8 & 2.67e0 & 2.67e0 & 2.67e0 & 2.67e0 \\
        \bottomrule
    \end{tabular}
    \caption{Mean-squared error on observed and unobserved entries for the
    structured instance in problem~\eqref{eq:structured_prob} with
    $n = 8$, $\ell = 3$, $\epsilon = 10^{-4}$, learning rate $0.02$, and
    initialization scale $10^{-4}$. Columns $s0$--$s9$ correspond to
    independent runs (\texttt{sol\_0}--\texttt{sol\_9}). 
    In most runs, $\mathrm{MSE}_{\text{obs}}$ is near machine precision while
    $\mathrm{MSE}_{\text{unobs}}$ is either of order $1$ (unsuccessful
    completions) or of order $10^{-8}$ (successful completions), highlighting
    that unsuccessful runs fit the observed entries but differ in their
    extrapolation to unobserved entries.}
    \label{tab:mse-obs-unobs}
\end{table}

\begin{figure}[H]
    \centering
    \subfigure[Run 1 (unsuccessful recovery)]{\label{fig:errmap-1}\includegraphics[width=0.48\linewidth]{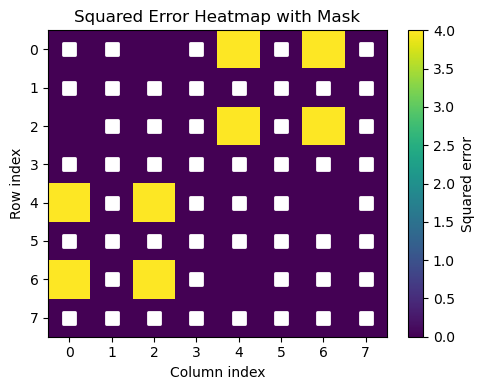}}\hfill
    \subfigure[Run 2 (unsuccessful recovery)]{\label{fig:errmap-2}\includegraphics[width=0.48\linewidth]{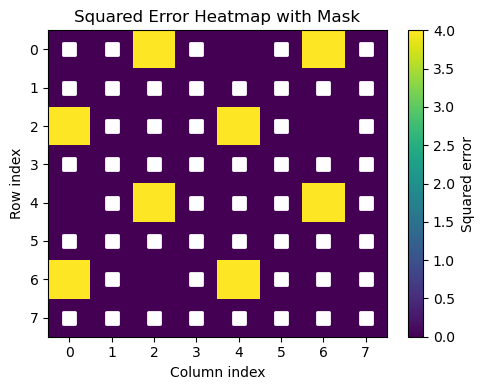}}

    \vspace{0.4em}

    \subfigure[Run 3 (unsuccessful recovery)]{\label{fig:errmap-3}\includegraphics[width=0.48\linewidth]{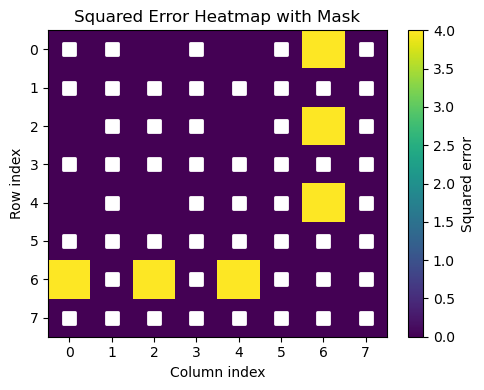}}\hfill
    \subfigure[Run 4 (successful recovery)]{\label{fig:errmap-4}\includegraphics[width=0.48\linewidth]{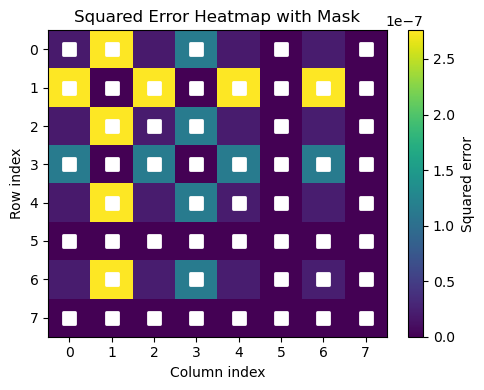}}

    \caption{Spatial distribution of squared reconstruction error for the
    structured instance in problem~\eqref{eq:structured_prob} with
    $n = 8$, $\ell = 3$, $\epsilon = 10^{-4}$, learning rate $0.02$, and
    initialization scale $10^{-4}$.
    The background colormap shows $(M^{\text{sol}}_{ij} - M^\star_{ij})^2$, and white
    squares denote observed entries $(i,j)\in\Omega$.
    In the three runs with higher global error (a)--(c), the reconstruction is
    essentially exact on observed entries and the error is concentrated on
    unobserved entries, with different spatial patterns across runs.
    In the successful run (d), the error is small on both observed and
    unobserved entries.}
    \label{fig:structured-error-maps}
\end{figure}


\newpage
\input{algorithmic_guarantees.tex}

\input{rebuttal_experiments.tex}

\end{document}

%% file: algorithmic_guarantees.tex
\section{Conditional algorithmic guarantees}\label{app:algorithmic_guarantees}

We make explicit how the pointwise escape result can support a perturbed-gradient-descent (pGD) argument, and the additional conditions needed. Throughout this section, $h^l$ has the unhalved normalization in \eqref{eq:lifted_main_highr}. For $\epsilon$-MC, $\alpha_s=\epsilon^2$, $L_s=1$, and $\tilde w=-\epsilon\A_{\Omc}\Mstar$; only its squared norm enters the following bounds.

\paragraph{Escape curvature.}
For a second-order point $\hat X$ covered by Theorem~\ref{thm:socp_lifted_noisy}, set
\[
 E:=\fro{\Mstar-\hat X\hat X^\top},\qquad
 s:=\lambda_r(\hat X\hat X^\top),\qquad
 a:=\frac{\alpha_s E^2-\|\tilde w\|_2^2}{2\tr(\Mstar)}.
\]
Assume $\tr(\Mstar)>0$, $a>0$, and odd $l$ with $\beta<2^{-1/l}$, where $\beta$ is defined in \eqref{eq:thm_socp_lcond_highr}; thus $L_s s=2a\beta$. The proof of Theorem~\ref{thm:socp_lifted_noisy} chooses unit vectors $u,q$ so that $U=uq^\top$ and $\Delta=\vecc(U)^{\otimes l}$ have unit Frobenius norm, $C_1=-G$ with $G\geq a$, and $C_2\leq L_s s/2$. Its half-squared-residual calculation gives
\begin{equation}\label{eq:algorithmic_curvature}
 2C_1^l+4C_2^l
 \leq-2G^l+2^{2-l}(L_s s)^l
 \leq-2a^l(1-2\beta^l)<0.
\end{equation}
Passing to the unhalved loss doubles the Hessian, so the negative bound in \eqref{eq:algorithmic_curvature} remains a conservative upper bound for $\lambda_{\min}(\nabla^2h^l(\mb{\hat w}))$ at $\mb{\hat w}=\vecc(\hat X)^{\otimes l}$.

\paragraph{From pointwise curvature to a region.}
Fix an ambient ball $K_R:=\{\mb w:\fro{\mb w}\leq R\}$. Since $h^l$ is polynomial, its gradient and Hessian have finite Lipschitz constants $L_h$ and $\rho_h$ on $K_R$. Let $\mathcal C_R\subseteq K_R$ contain the covered lifts above with uniform margins $a\geq\underline a>0$ and $\beta\leq\overline\beta<2^{-1/l}$, and set
\[
 \lambda:=2\underline a^l(1-2\overline\beta^l)>0.
\]
At any $\mb w\in K_R$ within Frobenius distance $\tau\leq\lambda/(2\rho_h)$ of $\mb w_1\in\mathcal C_R$, Weyl's inequality gives
\begin{equation}\label{eq:algorithmic_tube}
 \lambda_{\min}(\nabla^2h^l(\mb w))
 \leq-\lambda+\rho_h\fro{\mb w-\mb w_1}
 \leq-\lambda/2.
\end{equation}
Let $\mathcal P$ be the set of rank-one lifts where \eqref{eq:thm_socp_distance_lower_highr} fails. This protected set is not identified with the global minimizers. To obtain strict-saddle alternatives on a region $\mathcal D_R\subseteq K_R$, one additionally needs the following coverage condition: every $\mb w\in\mathcal D_R$ with $\fro{\nabla h^l(\mb w)}<\theta$ and $\operatorname{dist}(\mb w,\mathcal P)>\zeta$ lies in the tube in \eqref{eq:algorithmic_tube}. Theorem~\ref{thm:socp_lifted_noisy} establishes curvature at covered lifts of second-order points; it does not establish this condition at arbitrary small-gradient tensors.

\begin{corollary}[Conditional pGD conclusion]\label{cor:conditional_pgd}
Assume this coverage condition holds. Run the ambient-space pGD algorithm of \citet{jin2017escape}, with its prescribed perturbation distribution, step size, trigger, and waiting schedule. Assume its trajectory remains in $\mathcal D_R$ and all perturbed points and local segments required by its escape analysis remain in $K_R$. For an admissible tolerance
\[
 0<\eta_g<\min\{\theta,\lambda^2/(4\rho_h)\},
\]
with probability at least $1-\delta$, pGD reaches a point within distance $\zeta$ of $\mathcal P$ in $\widetilde{\mathcal O}(L_h\Delta_h/\eta_g^2)$ iterations, where $\Delta_h:=h^l(\mb w_0)-\inf_{\mb w}h^l(\mb w)$. The suppressed factors are polylogarithmic in dimension, inverse failure probability, and the landscape and accuracy parameters.
\end{corollary}
\begin{proof}
Under the stated smoothness and containment conditions, the pGD guarantee gives $\fro{\nabla h^l(\mb w)}\leq\eta_g$ and $\lambda_{\min}(\nabla^2h^l(\mb w))\geq-\sqrt{\rho_h\eta_g}>-\lambda/2$. The strict tolerance excludes both the large-gradient alternative and \eqref{eq:algorithmic_tube}; coverage therefore gives the claimed distance bound.
\end{proof}

Theorem~\ref{thm:implicit_bias_gd} concerns plain GD under small-initialization and finite-horizon conditions. Its spectral $\kappa$-rank-one property does not itself establish the Frobenius tube bound above or preserve that tube under random pGD perturbations. Similarly, the additional $\mathcal O(\sqrt r\,\kappa^{1/(2l)})$ recovery term in Theorem~\ref{thm:main_ext} applies under that theorem's separate hypotheses. Coverage and containment for the actual perturbed trajectory remain additional requirements; replacing ambient perturbations by rank-one-aware kicks requires a separate analysis. Consequently, Corollary~\ref{cor:conditional_pgd} supplies a conditional iteration statement, rather than a complete runtime and recovery theorem for GD followed by Tensor PCA and factor extraction.

%% file: rebuttal_experiments.tex

\section{Direct optimization of the surrogate}\label{app:direct_controls}
We separate the effect of perturbation from lifting by pairing ordinary MC
($\epsilon=0$) and $\epsilon$-MC ($\epsilon=5\times10^{-5}$) within each
solver, using the same problems, masks, and initializations.
The smaller controls use rank-one $n=8$ instances, Bernoulli sampling
with $p=0.15$ or the structured family in \eqref{eq:structured_prob}, and
$100$ runs per solver/family. Success means that the returned estimate has
Frobenius error at most $0.05$ relative to $M^*$.
For reference GNMR (Algorithm~3 of \citet{zilber2022gnmr}, $\alpha=-1$),
we evaluate its returned $UV^\top$ estimate.

Only $0.4\%$ of all paired binary success outcomes change.
The rank-five Bernoulli results at $n=200$ in Table~\ref{tab:direct_large}
use $10$ matched runs per cell. Across these $120$ pairs, only two outcomes
change, both for Adam. Perturbation alone therefore does not systematically
resolve the optimization difficulty in these tests. These controls are
separate from the original lifted-method comparisons in Figure~\ref{fig:success}.

The deterministic constants at this perturbation are
$\alpha_s=\epsilon^2=2.5\times10^{-9}$ and $L_s=1$, giving
$L_s/\alpha_s=4\times10^8$ and normalized RIP constant
$(1-\epsilon^2)/(1+\epsilon^2)\approx0.999999995$.
Thus positive RSC does not place the problem in the small-RIP regime of
conventional matrix-sensing landscape guarantees. This is a conditioning
diagnostic, not a causal explanation of every observed failure.

\clearpage
\section{Simulated lifting for the structured setting}\label{app:sod}
The direct-solver controls illustrate a practical distinction between
surrogate accuracy and optimization: a small perturbation can preserve an
accurate target while GD still converges to a spurious stationary point.
Tensor lifting addresses this obstruction by exposing negative curvature
at covered lifts of bad factorized points, including points that admit no
local descent direction in the original factor space. Explicitly optimizing
the lifted variable, however, requires $\mathcal O((nr)^l)$ storage. This
cost limits the dimensions that can be studied with the direct tensor
implementation even when the escape mechanism itself is useful.

Simulated Oracle Direction (SOD)~\citep{shen2026simulated} uses the lifted
construction to derive escape candidates that can be computed in the
original factor space. The higher-order tensor serves as a mathematical
oracle for constructing these candidates rather than as an optimization
variable that must be stored. When GD stalls, SOD evaluates the candidates
using the original objective, accepts an objective-decreasing move, and
resumes GD from the accepted point. This permits a finite escape move even
when a local factor-space search has stalled. The procedure can be repeated
at subsequent stationary points; its derivation and algorithmic details
are given by \citet{shen2026simulated}.

This connection suggests a practical route for using the lifted landscape
without its dominant storage cost. We examine that route on the structured
family in \eqref{eq:structured_prob}, augmenting ordinary GD with SOD while
keeping the same factor variable throughout. The study reaches
$n=50,100,200$ without forming the lifted tensor and compares the resulting
trajectories with GD under matched update budgets.

Let $M^*=zz^\top$, where $z_i=1$ for odd $i$ and $0$ otherwise. To examine
SOD in its noiseless weighted-sensing setting, this study uses
\begin{equation}\label{eq:sod_pilot}
f_n(x)=\frac12\|\mathcal A_\Omega(xx^\top-M^*)\|_2^2
+\frac{\epsilon_n^2}{2}
\|\mathcal A_{\Omega^c}(xx^\top-\mathbf1\mathbf1^\top)\|_2^2,
\end{equation}
with the all-ones vector $\mathbf1\in\mathbb R^n$ and
$\epsilon_n=0.48/\sqrt{\lceil n/2\rceil-1}$.
The complement target is fixed at one independently of the missing data;
in this structured family it coincides with the missing ground-truth
entries. Thus \eqref{eq:sod_pilot} uses the same observation pattern but a
different complement target from the zero-filled surrogate
\eqref{eq:eps_mc}. The scale $0.48$ was selected from
$\{0.18,0.24,0.30,0.36,0.42,0.48\}$ in separate $100$-run screening at
dimensions $20,25,30,35,40,45,50$; the reported trials use fresh seeds.

For each dimension, GD and GD+SOD share $500$ Gaussian initializations
normalized to Euclidean norm $10^{-2}$. The step size is
$\eta_n=\min\{0.02,1/n\}$, with a base cap of
$3000\max\{1,n/50\}$ updates and stationarity tolerance $10^{-8}$.
After the GD prefix, SOD is eligible when the gradient norm is at most
$2\times10^{-5}$. We allow at most five accepted SOD events, each followed
by at most $1200\max\{1,n/50\}$ GD updates. Plain GD receives exactly the
same number of GD updates as its paired GD+SOD run, while SOD adds the
escape computations. Success means $\|xx^\top-M^*\|_F\le0.05$.

\begin{table}[ht]
\centering
\caption{SOD success rates over $500$ paired starts per dimension, with
matched GD-update budgets.}
\label{tab:sod}
\begin{tabular}{rrrr}
\toprule
$n$ & $\epsilon_n$ & GD (\%) & GD+SOD (\%)\\
\midrule
50 & 0.097980 & 13.2 & 51.4\\
100 & 0.068571 & 2.0 & 46.4\\
200 & 0.048242 & 0.0 & 52.2\\
\bottomrule
\end{tabular}
\end{table}

Table~\ref{tab:sod} shows substantial improvements at all three dimensions,
with GD+SOD retaining roughly $46\%$--$52\%$ success through $n=200$.
These favorable results illustrate the applicability of simulated lifting
to the structured setting and motivate implementations that use the
lifted escape mechanism without explicit tensor storage.

%% file: neurips_2026.bbl
\begin{thebibliography}{}

\bibitem[Candes and Plan, 2010]{candes2010tight}
Candes, E.~J. and Plan, Y. (2010).
\newblock Tight oracle bounds for low-rank matrix recovery from a minimal
  number of random measurements.
\newblock {\em arXiv preprint arXiv:1001.0339}.

\bibitem[Cand{\`e}s and Recht, 2009]{candes2009exact}
Cand{\`e}s, E.~J. and Recht, B. (2009).
\newblock Exact matrix completion via convex optimization.
\newblock {\em Foundations of Computational Mathematics}, 9(6):717--772.

\bibitem[Cand{\`e}s and Tao, 2010]{candes2010power}
Cand{\`e}s, E.~J. and Tao, T. (2010).
\newblock The power of convex relaxation: Near-optimal matrix completion.
\newblock {\em IEEE Transactions on Information Theory}, 56(5):2053--2080.

\bibitem[Chen et~al., 2021]{chen2021spectral}
Chen, Y., Chi, Y., Fan, J., Ma, C., et~al. (2021).
\newblock Spectral methods for data science: A statistical perspective.
\newblock {\em Foundations and Trends{\textregistered} in Machine Learning},
  14(5):566--806.

\bibitem[Chen et~al., 2020]{chen2020noisy}
Chen, Y., Chi, Y., Fan, J., Ma, C., and Yan, Y. (2020).
\newblock Noisy matrix completion: Understanding statistical guarantees for
  convex relaxation via nonconvex optimization.
\newblock {\em SIAM journal on optimization}, 30(4):3098--3121.

\bibitem[Comon et~al., 2008]{comon2008symmetric}
Comon, P., Golub, G., Lim, L.-H., and Mourrain, B. (2008).
\newblock Symmetric tensors and symmetric tensor rank.
\newblock {\em SIAM Journal on Matrix Analysis and Applications},
  30(3):1254--1279.

\bibitem[Dai and Milenkovic, 2010]{dai2010set}
Dai, W. and Milenkovic, O. (2010).
\newblock Set: An algorithm for consistent matrix completion.
\newblock In {\em 2010 IEEE International Conference on Acoustics, Speech and
  Signal Processing}, pages 3646--3649. IEEE.

\bibitem[Ding and Chen, 2020]{ding2020leave}
Ding, L. and Chen, Y. (2020).
\newblock Leave-one-out approach for matrix completion: Primal and dual
  analysis.
\newblock {\em IEEE Transactions on Information Theory}, 66(11):7274--7301.

\bibitem[Du et~al., 2017]{du2017gradient}
Du, S.~S., Jin, C., Lee, J.~D., Jordan, M.~I., Singh, A., and Poczos, B.
  (2017).
\newblock Gradient descent can take exponential time to escape saddle points.
\newblock {\em Advances in neural information processing systems}, 30.

\bibitem[Fattahi and Sojoudi, 2020]{fattahi2020exact}
Fattahi, S. and Sojoudi, S. (2020).
\newblock Exact guarantees on the absence of spurious local minima for
  non-negative rank-1 robust principal component analysis.
\newblock {\em Journal of Machine Learning Research}, 21:1--51.

\bibitem[Ge et~al., 2017]{ge2017no}
Ge, R., Jin, C., and Zheng, Y. (2017).
\newblock No spurious local minima in nonconvex low rank problems: A unified
  geometric analysis.
\newblock In {\em Proceedings of the 34th International Conference on Machine
  Learning}, volume~70 of {\em Proceedings of Machine Learning Research}, pages
  1233--1242.

\bibitem[Ge et~al., 2016]{ge2016matrix}
Ge, R., Lee, J.~D., and Ma, T. (2016).
\newblock Matrix completion has no spurious local minimum.
\newblock {\em Advances in neural information processing systems}, 29.

\bibitem[Ghassemi et~al., 2018]{ghassemi2018global}
Ghassemi, M., Sarwate, A., and Goela, N. (2018).
\newblock Global optimality in inductive matrix completion.
\newblock In {\em 2018 IEEE International Conference on Acoustics, Speech and
  Signal Processing (ICASSP)}, pages 2226--2230. IEEE.

\bibitem[Gross, 2011]{gross2011recovering}
Gross, D. (2011).
\newblock Recovering low-rank matrices from few coefficients in any basis.
\newblock {\em IEEE Transactions on Information Theory}, 57(3):1548--1566.

\bibitem[Gu et~al., 2014]{gu2014weighted}
Gu, S., Zhang, L., Zuo, W., and Feng, X. (2014).
\newblock Weighted nuclear norm minimization with application to image
  denoising.
\newblock In {\em Proceedings of the IEEE conference on computer vision and
  pattern recognition}, pages 2862--2869.

\bibitem[Ha et~al., 2020]{ha2020equivalence}
Ha, W., Liu, H., and Barber, R.~F. (2020).
\newblock An equivalence between critical points for rank constraints versus
  low-rank factorizations.
\newblock {\em SIAM Journal on Optimization}, 30(4):2927--2955.

\bibitem[Haldar and Hernando, 2009]{haldar2009rank}
Haldar, J.~P. and Hernando, D. (2009).
\newblock Rank-constrained solutions to linear matrix equations using
  powerfactorization.
\newblock {\em IEEE Signal Processing Letters}, 16(7):584--587.

\bibitem[Jain and Kar, 2017]{jain2017nonconvex}
Jain, P. and Kar, P. (2017).
\newblock Non-convex optimization for machine learning.
\newblock {\em arXiv preprint arXiv:1712.07897}.

\bibitem[Jin et~al., 2017]{jin2017escape}
Jin, C., Ge, R., Netrapalli, P., Kakade, S.~M., and Jordan, M.~I. (2017).
\newblock How to escape saddle points efficiently.
\newblock In {\em International conference on machine learning}, pages
  1724--1732. PMLR.

\bibitem[Jin et~al., 2019]{jin2019towards}
Jin, M., Molybog, I., Mohammadi-Ghazi, R., and Lavaei, J. (2019).
\newblock Towards robust and scalable power system state estimation.
\newblock In {\em 2019 IEEE 58th Conference on Decision and Control (CDC)},
  pages 3245--3252. IEEE.

\bibitem[Keshavan et~al., 2010]{keshavan2010matrix}
Keshavan, R.~H., Montanari, A., and Oh, S. (2010).
\newblock Matrix completion from a few entries.
\newblock {\em IEEE transactions on information theory}, 56(6):2980--2998.

\bibitem[Klopp, 2015]{klopp2015matrix}
Klopp, O. (2015).
\newblock Matrix completion by singular value thresholding: sharp bounds.
\newblock {\em Electronic Journal of Statistics}, 9(2):2348--2369.

\bibitem[Kolda, 2015]{kolda2015numerical}
Kolda, T.~G. (2015).
\newblock Numerical optimization for symmetric tensor decomposition.
\newblock {\em Mathematical Programming}, 151(1):225--248.

\bibitem[Koren et~al., 2009]{koren2009matrix}
Koren, Y., Bell, R., and Volinsky, C. (2009).
\newblock Matrix factorization techniques for recommender systems.
\newblock {\em Computer}, 42(8):30--37.

\bibitem[Lasserre, 2001]{lasserre2001global}
Lasserre, J.~B. (2001).
\newblock Global optimization with polynomials and the problem of moments.
\newblock {\em SIAM Journal on optimization}, 11(3):796--817.

\bibitem[Lee and Bresler, 2010]{lee2010admira}
Lee, K. and Bresler, Y. (2010).
\newblock Admira: Atomic decomposition for minimum rank approximation.
\newblock {\em IEEE Transactions on Information Theory}, 56(9):4402--4416.

\bibitem[Li et~al., 2019]{li2019non}
Li, Q., Zhu, Z., and Tang, G. (2019).
\newblock The non-convex geometry of low-rank matrix optimization.
\newblock {\em Information and Inference: A Journal of the IMA}, 8(1):51--96.

\bibitem[Li et~al., 2016]{LLR16}
Li, Y., Liang, Y., and Risteski, A. (2016).
\newblock Recovery guarantee of weighted low-rank approximation via alternating
  minimization.
\newblock In {\em Proceedings of the 33rd International Conference on Machine
  Learning}, volume~48 of {\em Proceedings of Machine Learning Research}, pages
  2358--2367. PMLR.

\bibitem[Li et~al., 2018]{li2018algorithmic}
Li, Y., Ma, T., and Zhang, H. (2018).
\newblock Algorithmic regularization in over-parameterized matrix sensing and
  neural networks with quadratic activations.
\newblock In {\em Conference On Learning Theory}, pages 2--47. PMLR.

\bibitem[Lustig et~al., 2008]{lustig2008compressed}
Lustig, M., Donoho, D.~L., Santos, J.~M., and Pauly, J.~M. (2008).
\newblock Compressed sensing mri.
\newblock {\em IEEE signal processing magazine}, 25(2):72--82.

\bibitem[Ma and Fattahi, 2022]{ma2022global}
Ma, J. and Fattahi, S. (2022).
\newblock Global convergence of sub-gradient method for robust matrix recovery:
  Small initialization, noisy measurements, and over-parameterization.
\newblock {\em arXiv preprint arXiv:2202.08788}.

\bibitem[Ma et~al., 2024]{ma2024algorithmic}
Ma, Z., Lavaei, J., and Sojoudi, S. (2024).
\newblock Algorithmic regularization in tensor optimization: Towards a lifted
  approach in matrix sensing.
\newblock {\em Advances in Neural Information Processing Systems}, 36.

\bibitem[Ma et~al., 2023]{ma2023over}
Ma, Z., Molybog, I., Lavaei, J., and Sojoudi, S. (2023).
\newblock Over-parametrization via lifting for low-rank matrix sensing:
  Conversion of spurious solutions to strict saddle points.
\newblock In {\em International Conference on Machine Learning}. PMLR.

\bibitem[Mishra et~al., 2014]{mishra2014fixed}
Mishra, B., Meyer, G., Bonnabel, S., and Sepulchre, R. (2014).
\newblock Fixed-rank matrix factorizations and riemannian low-rank
  optimization.
\newblock {\em Computational Statistics}, 29:591--621.

\bibitem[Mohan and Fazel, 2012]{mohan2012iterative}
Mohan, K. and Fazel, M. (2012).
\newblock Iterative reweighted algorithms for matrix rank minimization.
\newblock {\em The Journal of Machine Learning Research}, 13(1):3441--3473.

\bibitem[Molybog et~al., 2020]{molybog2020conic}
Molybog, I., Madani, R., and Lavaei, J. (2020).
\newblock Conic optimization for quadratic regression under sparse noise.
\newblock {\em The Journal of Machine Learning Research}, 21(1):7994--8029.

\bibitem[Nguyen et~al., 2019]{nguyen2019low}
Nguyen, L.~T., Kim, J., and Shim, B. (2019).
\newblock Low-rank matrix completion: A contemporary survey.
\newblock {\em IEEE Access}, 7:94215--94237.

\bibitem[Parrilo, 2003]{parrilo2003semidefinite}
Parrilo, P.~A. (2003).
\newblock Semidefinite programming relaxations for semialgebraic problems.
\newblock {\em Mathematical programming}, 96(2):293--320.

\bibitem[Petersen et~al., 2008]{petersen2008matrix}
Petersen, K.~B., Pedersen, M.~S., et~al. (2008).
\newblock The matrix cookbook.
\newblock {\em Technical University of Denmark}, 7(15):510.

\bibitem[Pimentel-Alarc{\'o}n et~al., 2016]{pimentel2016characterization}
Pimentel-Alarc{\'o}n, D.~L., Boston, N., and Nowak, R.~D. (2016).
\newblock A characterization of deterministic sampling patterns for low-rank
  matrix completion.
\newblock {\em IEEE Journal of Selected Topics in Signal Processing},
  10(4):623--636.

\bibitem[Radhakrishnan et~al., 2025]{radhakrishnan2025linear}
Radhakrishnan, A., Belkin, M., and Drusvyatskiy, D. (2025).
\newblock Linear recursive feature machines provably recover low-rank matrices.
\newblock {\em Proceedings of the National Academy of Sciences},
  122(13):e2411325122.

\bibitem[Razenshteyn et~al., 2016]{RSW16}
Razenshteyn, I., Song, Z., and Woodruff, D.~P. (2016).
\newblock Weighted low rank approximations with provable guarantees.
\newblock In {\em Proceedings of the 48th Annual ACM Symposium on Theory of
  Computing}, pages 250--263. ACM.

\bibitem[Recht, 2011]{recht2011simpler}
Recht, B. (2011).
\newblock A simpler approach to matrix completion.
\newblock {\em Journal of Machine Learning Research}, 12(12).

\bibitem[Recht et~al., 2010]{recht2010guaranteed}
Recht, B., Fazel, M., and Parrilo, P.~A. (2010).
\newblock Guaranteed minimum-rank solutions of linear matrix equations via
  nuclear norm minimization.
\newblock {\em SIAM Review}, 52(3):471--501.

\bibitem[Shen et~al., 2026]{shen2026simulated}
Shen, T., Yang, J., He, J., Gao, K., Zheng, Z., and Ma, Z. (2026).
\newblock Escaping local minima provably in non-convex matrix sensing: A
  deterministic framework via simulated lifting.
\newblock {\em arXiv preprint arXiv:2602.05887}.

\bibitem[Tanner and Wei, 2016]{tanner2016low}
Tanner, J. and Wei, K. (2016).
\newblock Low rank matrix completion by alternating steepest descent methods.
\newblock {\em Applied and Computational Harmonic Analysis}, 40(2):417--429.

\bibitem[Wang et~al., 2014]{wang2014rank}
Wang, Z., Lai, M.-J., Lu, Z., Fan, W., Davulcu, H., and Ye, J. (2014).
\newblock Rank-one matrix pursuit for matrix completion.
\newblock In {\em International Conference on Machine Learning}, pages 91--99.
  PMLR.

\bibitem[Wen et~al., 2012]{wen2012solving}
Wen, Z., Yin, W., and Zhang, Y. (2012).
\newblock Solving a low-rank factorization model for matrix completion by a
  nonlinear successive over-relaxation algorithm.
\newblock {\em Mathematical Programming Computation}, 4(4):333--361.

\bibitem[Yalcin et~al., 2023]{yalcin2022semidefinite}
Yalcin, B., Ma, Z., Lavaei, J., and Sojoudi, S. (2023).
\newblock Semidefinite programming versus burer-monteiro factorization for
  matrix sensing.
\newblock In {\em Proceedings of the AAAI Conference on Artificial
  Intelligence}.

\bibitem[Yal{\c{c}}{\i}n et~al., 2022]{yalccin2022factorization}
Yal{\c{c}}{\i}n, B., Zhang, H., Lavaei, J., and Sojoudi, S. (2022).
\newblock Factorization approach for low-complexity matrix completion problems:
  Exponential number of spurious solutions and failure of gradient methods.
\newblock In {\em International Conference on Artificial Intelligence and
  Statistics}, pages 319--341. PMLR.

\bibitem[Zhang et~al., 2021]{zhang2021general}
Zhang, H., Bi, Y., and Lavaei, J. (2021).
\newblock General low-rank matrix optimization: Geometric analysis and sharper
  bounds.
\newblock {\em Advances in Neural Information Processing Systems},
  34:27369--27380.

\bibitem[Zhang, 2021]{zhang2021sharp}
Zhang, R.~Y. (2021).
\newblock Sharp global guarantees for nonconvex low-rank matrix recovery in the
  overparameterized regime.
\newblock {\em arXiv preprint arXiv:2104.10790}.

\bibitem[Zhang, 2022]{zhang2022improved}
Zhang, R.~Y. (2022).
\newblock Improved global guarantees for the nonconvex burer--monteiro
  factorization via rank overparameterization.
\newblock {\em arXiv preprint arXiv:2207.01789}.

\bibitem[Zhang et~al., 2017]{zhang2017conic}
Zhang, Y., Madani, R., and Lavaei, J. (2017).
\newblock Conic relaxations for power system state estimation with line
  measurements.
\newblock {\em IEEE Transactions on Control of Network Systems},
  5(3):1193--1205.

\bibitem[Zilber and Nadler, 2022a]{zilber2022gnmr}
Zilber, P. and Nadler, B. (2022a).
\newblock Gnmr: A provable one-line algorithm for low rank matrix recovery.
\newblock {\em SIAM journal on mathematics of data science}, 4(2):909--934.

\bibitem[Zilber and Nadler, 2022b]{zilber2022inductive}
Zilber, P. and Nadler, B. (2022b).
\newblock Inductive matrix completion: No bad local minima and a fast
  algorithm.
\newblock In {\em International Conference on Machine Learning}, pages
  27671--27692. PMLR.

\end{thebibliography}
